\documentclass[12pt]{article}
\usepackage[english]{babel}
\usepackage{amssymb,amsmath,graphicx}
\usepackage{amsthm}
\usepackage{float}
\usepackage{afterpage}
\usepackage{dcolumn}
\usepackage{multirow}
\usepackage{color}
\usepackage[round,authoryear]{natbib}
\usepackage{enumitem}
\usepackage{slashbox}
\usepackage{pict2e}
 \usepackage{setspace}
\usepackage{lscape}

\usepackage[noend]{algorithmic}
\usepackage{algorithm}
\usepackage{rotating}

\newcommand{\mycomment}[1]{}

\usepackage{array}
\usepackage{url}
\usepackage{pst-all}
\usepackage{appendix}

\newtheorem{proposition}{Proposition}

\newcolumntype{d}[1]{D{.}{.}{#1}}

\DeclareGraphicsExtensions{.png}
\newcolumntype{P}[1]{>{\centering}p{#1}}
\newcolumntype{M}[1]{>{\centering}m{#1}}
\usepackage{color}

\usepackage{flafter}

\def\sqw{\hbox{\rlap{\leavevmode\raise.3ex\hbox{$\sqcap$}}$%
\sqcup$}}
\def\sqb{\hbox{\hskip5pt\vrule width4pt height6pt depth1.5pt%
\hskip1pt}}

\def\qed{\ifmmode\hbox{\hfill\sqb}\else{\ifhmode\unskip\fi%
\nobreak\hfil
\penalty50\hskip1em\null\nobreak\hfil\sqb
\parfillskip=0pt\finalhyphendemerits=0\endgraf}\fi}
\def\cqfd{\ifmmode\sqw\else{\ifhmode\unskip\fi\nobreak\hfil
\penalty50\hskip1em\null\nobreak\hfil\sqw
\parfillskip=0pt\finalhyphendemerits=0\endgraf}\fi}

\usepackage{natbib}

 \bibpunct[, ]{(}{)}{,}{a}{}{,}%
 \def\newblock{\ }%

\def\Our{CluVRP}
\def\Gen{UHGS}
\def\ILS{ILS}
\def\ILSClu{ILS-Clu}

\usepackage[noend]{algorithmic}
\usepackage{algorithm}

\DeclareGraphicsExtensions{.png}

\newcolumntype{M}[1]{>{\flushleft}m{#1}}

\title{Hybrid Metaheuristics for the Clustered Vehicle Routing Problem}

\author{Thibaut Vidal, Nelson Maculan, Puca Huachi Vaz Penna, Luis Satoru Ochi}

\begin{document}

\begin{center}

\begin{LARGE}
Hybrid Metaheuristics for the Clustered Vehicle Routing Problem
\end{LARGE}

\vspace*{0.65cm}

\textbf{Thibaut Vidal} \\
LIDS, Massachusetts Institute of Technology \\
vidalt@mit.edu \\
\vspace*{0.2cm}
\textbf{Maria Battarra} \\
University of Southampton, School of Mathematics, U.K.  \\
m.battarra@soton.ac.uk \\
\vspace*{0.2cm}
\textbf{Anand Subramanian} \\
Universidade Federal da Para\'{i}ba - Departamento de Engenharia de Produ\c{c}\~{a}o, Brazil  \\
anand@ct.ufpb.br \\
\vspace*{0.2cm}
\textbf{G\"{u}ne\c{s} Erdo\v{g}an} \\
University of Southampton, School of Management, U.K. \\
g.erdogan@soton.ac.uk \\

\vspace*{1.0cm}

\begin{large}
Working Paper, MIT -- April 2014
\end{large}

\vspace*{0.7cm}

\end{center}
\noindent
\textbf{Abstract.}
The Clustered Vehicle Routing Problem (\Our) is a variant of the Capacitated Vehicle Routing Problem in which customers are grouped into clusters. Each cluster has to be visited once, and a vehicle entering a cluster cannot leave it until all customers have been visited. This article presents two alternative hybrid metaheuristic algorithms for the {\Our}. The first algorithm is based on an Iterated Local Search algorithm, in which only feasible solutions are explored and problem-specific local search moves are utilized. The second algorithm is a Hybrid Genetic Search, for which the shortest Hamiltonian path between each pair of vertices within each cluster should be precomputed. 
Using this information, a sequence of clusters can be used as a solution representation and large neighborhoods can be efficiently explored by means of bi-directional dynamic programming, sequence concatenations, by using appropriate data structures. Extensive computational experiments are performed on benchmark instances from the literature, as well as new large scale ones. Recommendations on promising algorithm choices are provided relatively to average cluster size. 

\vspace*{0.6cm}


\vspace*{0.7cm}

\noindent

\newpage

\section{Introduction}

This paper addresses the \emph{Clustered Vehicle Routing Problem} ({\Our}), which has been recently introduced to the literature by \cite{SevauxS2008}. 
The \Our\ is defined over an undirected graph $\mathcal{G} = (\mathcal{V}, \mathcal{E})$, where the vertex 0 is the depot and any other vertex $i \in \mathcal{V}\setminus\{0\}$ is a customer with demand $q_i>0$. A fleet of $m$ vehicles, each with capacity $Q$, is stationed at the depot. The set of customers is partitioned into $N$ non-intersecting and nonempty subsets called \emph{clusters}, such that $\mathcal{V} = V_1 \cup \dots \cup V_N$. The customers in each cluster have to be visited consecutively, such that the vehicle visiting a customer in the cluster cannot leave the cluster until all the therein customers have not been visited. Each edge $(i,j) \in \mathcal{E}$ is associated with a travel cost $c_{ij}$, and the objective is to minimize the total travel cost. 
The Capacitated Vehicle Routing Problem (CVRP) is a special case of the \Our\ in which each vertex is a cluster on its own. Since the CVRP is  $\cal{NP}$-Hard, the \Our\ is also $\cal{NP}$-Hard. 

\cite{SevauxS2008} introduced the \Our\ in the context of a real-world application where containers are employed to carry goods. The customers expecting parcels in the same container form a cluster, because the courier has to deliver the content of a whole container before handling another container. Clusters also arise in applications involving passengers transportation, where passengers prefer to travel with friends or neighbors (as in the transportation of elderly to recreation centres). Gated communities (residential or industrial areas enclosed in walled enclaves for safety and protection reasons) provide another natural example of clusters.  The customers within a gated community are likely to be visited by a single vehicle in a sequence, otherwise the vehicles have to spend additional time for the security controls at the gates. 

Clusters can thus be imposed by the geography, the nature of the application, as well as by practitioners aiming to achieve \emph{compact} and easy-to-implement routing solutions. Clustered routes allow drivers to be assigned to areas (i.e., certain streets or postcodes) and allow the development of familiarity, which makes their task easier. In addition, clustered routes do not remarkably overlap among each other. In several cases, the additional routing costs due to cluster constraints are compensated by the ease of implementation and the enhanced driver familiarity.
 
The literature on the {\Our} is quite limited as of the time of this writing. \cite{Sorensen08} and \cite{SevauxS2008} presented an integer programming formulation capable of finding the best Hamiltonian path for each pair of vertices in each cluster. \cite{Barthelemy10} suggested to adapt CVRP algorithms to the \Our\ by including a large positive term $M$ to the cost of the edges between clusters and a cluster and the depot. The \Our\ is solved as a CVRP by means of the algorithm of \citet{ClarkeW1979} followed by \textsc{2-opt} moves and Simulated Annealing (SA).
The authors also suggested to dynamically set the penalty $M$, but observed that the $M$ term interferes with the Boltzmann acceptance criterion of the SA and leads to erratic performance. Computational results were not reported in this initial paper.

\cite{PopKH2012} described the directed \Our\ as an extension of the \emph{Generalized Vehicle Routing Problem} (GVRP)(\citealp{GhianiI2000}). The authors adapted two polynomial-sized formulations for the GVRP to the directed \Our, but again no computational results were reported. Recently, \cite{BattarraEV2014} proposed exact algorithms for the \Our\ and provided a set of benchmark instances with up to 481 vertices. The best performing algorithm relies on a preprocessing scheme, in which the best Hamiltonian path is precomputed for each pair of endpoints in each cluster. This allows for selecting a pair of endpoints in each cluster rather than the whole path, relegating some of the problem complexity in the preprocessing scheme. 
The resulting minimum cost Hamiltonian path problems are reduced to instances of the Traveling Salesman Problem (TSP) and optimally solved with Concorde (\citealp{concorde}). Instances of much larger size than the corresponding CVRP instances were optimally solved, thus highlighting the advantage of acknowledging the presence of clusters.

In this paper, we introduce new \mycomment{hybrid} adaptations of state-of-the-art CVRP metaheuristics for the \Our. Rather than rediscovering well-known metaheuristic concepts, we exploit the current knowledge on iterated local search and hybrid genetic algorithms \citep{Subramanian2012,Vidal2012b} and focus our attention on developing efficient problem-tailored neighborhood searches and effectively embedding them into these metaheuristic frameworks. The proposed neighborhood searches aim at 1) better exploiting clustering constraints by means of pruning techniques, 2) exploring larger neighborhoods by means of dynamic programming, 3) reducing the computational time by means of re-optimization, bi-directional search, and data structures. Finally, these experiments lead to further insights on which type of metaheuristic to use for different instance sizes and cluster characteristics. 

The reminder of the paper is organized as follows. Section \ref{Sec.Motivation} introduces the challenges related to the \Our. Sections \ref{SecILS} and \ref{SecGen} describe the proposed metaheuristics and efficient neighborhood-search strategies, whereas Section \ref{Sec.CompResults} discusses our computational results. Conclusions are drawn in Section \ref{SecConclusions}, and further avenues of research are discussed. 

\section{Motivation}\label{Sec.Motivation}

\cite{BattarraEV2014} showed that exact algorithms are capable of solving relatively large \Our\ instances. However, the CPU times remain prohibitively long for large-scale or real time applications. In this paper, we exploit the properties of the \Our\ to develop specialized \mycomment{hybrid} metaheuristics that take advantage of cluster constraints. Solution quality is assessed by a comparison with exact solutions whenever possible, and among metaheuristics when it is not.


Two recent and successful metaheuristic frameworks are used in this work.
The \ILS\ algorithm of \cite{Subramanian2012} is simple and flexible, combining the intensification strength of Local Search (LS) operators and effective diversification through perturbation operators. 
It proved to be remarkably efficient for many variants of the Vehicle Routing Problem (VRP), including the VRP with Simultaneous Pickup and Delivery (\citealp{SubramanianDBOF2010}), the Heterogeneous VRP (\citealp{PennaSO2013}), the Minimum Latency Problem (\citealp{Silva2012}) and the TSP with Mixed Pickup and Delivery (\citealp{SubramanianB2013}). The success of \ILS\ is due to a clever design of intensification and diversification neighbourhoods, as well as their random exploration. This latter component allows for extra diversity, and leads to high quality solutions, even when applied to other problems such as scheduling (\citealp{SubramanianBP2014}).

\ILS\ explores only feasible solutions, and allows for testing the $M$ approach suggested by \cite{Barthelemy10} without possible interferences between $M$ and penalties applied to infeasible solutions. As mentioned in the introduction, the $M$ approach consists of including a large positive term to all those edges that are connecting clusters and connecting the depot to the clusters. Any CVRP algorithm in which the $M$ is chosen to be large enough returns a \Our\ solution in which the number of penalized edges is minimized, therefore a solution in which the cluster constraint is satisfied. Note that the number of edges connecting clusters or connecting the depot to a cluster is $m+N$ and their penalization can be easily deducted from the solution cost. 

One drawback of this transformation is that most VRP neighborhoods consider moves of one or two vertices. These neighborhoods can often not relocate complete clusters, and thus many moves appear largely deteriorating due to $M$ penalties, significantly inhibiting the progress towards higher quality solutions. As shown in this paper, ILS can partly overcome this issue by means of shaking moves. However, as demonstrated by our computational results, a more clever application of the framework specific to the CluVRP considering relocate and exchanges of full clusters and intra-cluster improvements produces solutions of comparable quality in considerably less CPU time. In the next section, we describe the \ILS\ and these new \mycomment{hybrid algorithms} in more details.

The Unified Hybrid Genetic Search (\Gen) currently obtains the best known solutions for more than 30 variants of the CVRP and represents the state-of-the-art among hybrid metaheuristics for vehicle routing problems. More precisely, the algorithm succesfully solves problems with diverse attributes, such as multiple depots and periods \citep{Vidal2012}, time windows and vehicle-site dependencies \citep{Vidal2012c}, hours-of-service-regulations for various countries \citep{Goel2012}, soft, multiple, and general time windows, backhauls, asymmetric, cumulative and load-dependent costs, simultaneous pickup and delivery, fleet mix, time dependency and service site choice \citep{Vidal2012b}, and prize-collecting problems \citep{Vidal2014}, among others. It has been recently demonstrated that several combinatorial decisions, such as customer selections or depot placement, can be relegated directly at the level of cost and route evaluations, allowing to always rely on the same metaheuristic and local search framework while exploring large neighborhoods in polynomial or pseudo-polynomial time \citep{Vidal2012e,Vidal2014}.

Our \Gen\ implementation is based on the assumption that the costs of the optimal Hamiltonian paths among vertices in the same cluster can be efficiently precomputed as in \cite{BattarraEV2014}. Once these paths and their costs are known, an effective route representation as an ordered sequence of clusters can be adopted, and a fast shortest path-based algorithm for converting this solution representation into the corresponding optimal sequence of customers is presented in Section \ref{SecGen}. This drastically reduces the size of the search space of the \Gen\ method, which optimizes the assignment and sequencing of $\mathcal{O}(N)$ clusters instead of $\mathcal{O}(n)$ customers.

Our computational experiments allow to quantify the trade-off between adopting the preprocessing scheme to compute the Hamiltonian paths, which requires the solution of $\sum_{i=1, \dots, N} |V_i|\times (|V_i|-1)$ TSP instances and searching in the space of clusters with \Gen, or working in the space of vertices with a well-designed \ILS. As long as the average size of the clusters is not high, the computational burden of the preprocessing is not prohibitive, but is observed to become significant when the cluster size increases. On the other hand, \Gen\ is much faster when the preprocessing information is known and obtains higher quality solutions. Through our computational experiments, we aim at identifying a critical cluster size that makes an approach with cluster-based solution representation more desirable than an approach using vertex-based representation.

\section{The \ILS\ metaheuristic} \label{SecILS}

As previously mentioned, the algorithm of \cite{Subramanian2012} can be used for solving the \Our\ by applying suitable penalties to edges between clusters and between clusters and the depot. Although simple, this straightforward adaptation has two main drawbacks: (i) most of the local search moves violate the cluster constraint, leading to high penalties, and consuming a large part of the CPU time; and (ii) many promising moves that relocate full clusters are not included in the neighborhoods, thus reducing the intensification capabilities of the LS. \ILS\ was therefore adapted to better take advantage of clusters. In the following, we denote this adaptation as \ILSClu. 

The \ILSClu\ is a hybrid algorithm built upon the structure of ILS. Large neighborhoods proved to be very effective in solving VRP variants, however, identifying promising moves can be a difficult task that is usually left for a large part to randomization (e.g., in Adaptive Large Neighborhood Search, \citealp{PisingerR2007}). In contrast, the CluVRP structure enables to apply moves to relevant sets of customers. Thus, the LS phase of \ILSClu\ explores moves on different levels: among clusters, among edges connecting clusters or clusters with the depot, and within each cluster. This mechanism enables to explore a larger variety of moves while significantly reducing CPU time.

The \ILS\ of \cite{Subramanian2012} is a multi-start heuristic which returns the best solution after $n_R$ restarts. Each iteration is finished when $n_I$ consecutive shaking phases without improvement are attained.
The initial solution is generated using a \emph{parallel cheapest insertion heuristic}. Iteratively, a randomly selected customer is inserted with minimum cost, either between customers from the same cluster, or between two clusters.


Both \ILS\ and \ILSClu\ apply a perturbation mechanism after each local search phase, which consists of one or two randomly selected \textsc{Shift(1,1)} or \textsc{Swap} moves. In \ILS, \textsc{Shift(1,1)} relocates a random customer from its route $r$ to a random position in another route $r'$, and simultaneously relocates a random customer from $r'$ to a random position in $r$. The same process is applied in \ILSClu\, but considering clusters instead of single customers. Moreover, in \ILS, \textsc{Swap} exchanges two customers from different routes, whereas in \ILSClu\, the exchange involves two clusters of the same route.
Inter-route LS neighborhoods are first applied in a random order, and intra-route LS operators are employed in a random order whenever an improving solution is found.

\begin{algorithm}[ht]
  \caption{ILS}
\label{ILS}
\begin{algorithmic}[1]
\STATE \textbf{Procedure ILS:}
\STATE $s_{0}$ $\leftarrow$ GenerateInitialSolution;
\STATE $s^{*}$ $\leftarrow$ LocalSearch($s_{0}$);
\WHILE {Stopping criterion is not met} 
\STATE $s^{\prime}$ $\leftarrow$ Perturb($s^{*}$, history);
\STATE $s^{*\prime}$ $\leftarrow$ LocalSearch($s^{\prime}$);
\STATE $s^{*}$ $\leftarrow$ AcceptanceCriterion($s^{*}$, $s^{*\prime}$, history);	 
\ENDWHILE
\STATE \textbf{end} ILS;
\end{algorithmic}
\end{algorithm}

\begin{algorithm}[ht]
\caption{LS os ILS and ILS of ILS-Clu}	 
\label{ILS2}
\begin{algorithmic}
\begin{minipage}{0.47\textwidth}
\STATE \textbf{Local Search of ILS:}
\STATE Init inter-route Neighborhood List (NL);
\WHILE {NL $\not= 0$}
	\STATE Choose random Neighborhood $\in$ NL;
   \STATE Find best  $s^\prime$ of $s \in$ Neighborhood;
   	\IF {$f(s^\prime) < f(s)$}
			\STATE $s \leftarrow s^\prime$;
			\STATE $s \leftarrow $ IntraRouteSearch($s$);
                        \STATE Update NL; 
			\STATE			
			\STATE
                        \STATE
			\STATE
			
		\ELSE
			\STATE Remove Neighborhood from NL;
		\ENDIF
\ENDWHILE  
\STATE
\STATE \textbf{return} $s$;
\STATE \textbf{end}.      
\end{minipage}
\hspace*{0.2cm}
\begin{minipage}{0.47\textwidth}
\STATE \textbf{Local Search of ILS-Clu:}
\STATE Init inter-route Neighborhood List (NL$_C$);
\WHILE {NL$_C$ $\not= 0$}
	\STATE Choose random Neighborhood $\in$ NL$_C$;
   \STATE Find best $s^\prime$ of $s \in$ Neighborhood;
   	\IF {$f(s^\prime) < f(s)$}
   			\STATE $s^{\prime} \leftarrow$ EndPointsSearch($s^\prime$);
			\STATE $\bar{s} \leftarrow $ IntraRouteClusterSearch($s^\prime$);
			\IF {$f(\bar{s}) < f(s)$}
				\STATE $s \leftarrow$ EndPointsSearch($\bar{s}$);
			\ELSE 	\STATE  $s \leftarrow \bar{s}$;
			\ENDIF
			\STATE Update NL$_C$; 
	\ELSE
			\STATE Remove Neighborhood from NL$_C$;
		\ENDIF
\ENDWHILE
\STATE $s \leftarrow $ IntraClusterSearch($s$);
\STATE \textbf{return} $s$;
\STATE \textbf{end}.   
\end{minipage}
\end{algorithmic} 
\end{algorithm}

Algorithm \ref{ILS2} presents the common structure of both iterated local search algorithms. Algorithm \ref{ILS2} highlights the differences between the LS stage of \ILS\ and \ILSClu. The neighbourhoods $NL_C$, as well as the operators implemented in ``IntraRouteClusterSearch'' within \ILSClu\ modify the sequence of clusters in the routes without changing the endpoints or the Hamiltonian paths in each cluster. To partially remedy this myopic strategy, the operator ``EndPointsSearch'' aims at selecting the most effective endpoints in the clusters whenever an improving move is found. Note that ``EndPointsSearch'' does not modify the sequence of customers visited within each cluster. 
The neighborhoods considered in NL$_C$ are \textsc{Relocate1}, \textsc{Relocate2}, \textsc{Swap(1,1)}, \textsc{Swap(2,1)} and \textsc{Swap(2,2)}, as well as \textsc{2-opt*}. In  ``IntraRouteClusterSearch'', the neighborhoods are \textsc{Or-opt},  \textsc{2-opt} and \textsc{Swap}. These neighborhoods are considered in random order. Detailed descriptions of these families of neighborhoods can be found in \cite{Subramanian2012} and \cite{Vidal2012a}. ``EndPointsSearch'' and  ``IntraRouteClusterSearch'' sequentially search for improving \textsc{Relocate}, \textsc{2-opt} and \textsc{Swap} moves within the clusters. The moves considered in ``EndPointsSearch'' is a subset of ``IntraRouteClusterSearch'', in which at least one customer involved in the move is currently serviced first or last in its cluster.


\section{The \Gen\ metaheuristic} \label{SecGen}

\Gen\ is a successful framework capable of producing high quality solutions for many VRP variants. It is a hybrid algorithm, where the diversification strength of a Genetic Algorithm (GA) is combined with the fast improvement capabilities of local search. One main challenge in the design of a hybrid genetic algorithm is to achieve a good balance between intensification and diversification while controlling the use of computationally intensive local search procedures.
This balance is usually achieved by selecting a suitable initial population, crossover operators, mutation, and selection mechanisms. The variety of design choices and the tuning of a multitude of parameters often inhibit the flexibility of the GAs. In fact, most of the previous attempts in the literature focused on the design of problem-specific operators, failing to lead to general algorithms and frequently resulting in a large number of parameters to be tuned. \Gen\ \citep{Vidal2012b} managed to overcome most of these drawbacks by adopting the following strategies. 

\subsection{General UHGS methodology}

 \Gen\ evolves a population of individuals representing problem solutions, by means of selection, crossover and education operators. Note that the operator \emph{education} involves a complete local-search procedure aimed at improving the solutions rather than a randomized mutation. The population is managed to contain between $\mu^\textsc{min}$ and $\mu^\textsc{min} + \mu^\textsc{gen}$ individuals, by pruning $\mu^\textsc{gen}$ individuals whenever the maximum size is attained.
The method is run until $It_{max}$ individuals have been successively created without improvement of the best solution.

\Gen\ achieves a fine balance between intensification and diversification by means of a bi-criteria evaluation of solutions. The first criterion is the contribution of a solution to the population diversity, which is measured as the Hamming distance of the solution to the closest solutions in the population. The second criterion is the objective value. Solutions are ranked with respect to both criteria, and the sum of the ranks provides a ``biased fitness'' \citep{Vidal2012b}, used for both parents selection and survivors selection when the maximum population size is attained. To deal with tightly constrained problems, linearly penalized route-constraint violations -- capacity or distance -- are included in the objective. Penalty coefficients are dynamically adjusted to ensure a target ratio of naturally-feasible solutions during the search, and infeasible solutions are managed in a secondary population.

During crossover, the whole solution is represented as a giant tour visiting all customers once, without intermediate depot trips. As such, a simple ordered crossover (OX) that works on permutations can be used. The optimal splitting of the giant tour into separate routes is performed optimally in polynomial time as a shortest path subproblem on an auxiliary graph \citep{Prins2004}. This process is known to be widely applicable in a unified manner to many vehicle routing variants as long as it is possible to perform separate efficient route evaluations to compute the cost of edges in the auxiliary graph \citep{Vidal2013}. Finally, \Gen\ relies on local search to improve every new offspring solution generated during the search. The LS operators used in \Gen\ are \textsc{2-opt}, \textsc{2-opt*},  \textsc{Cross} and  \textsc{I-Cross} (\citealp{Vidal2012b}), limited to sequences of less than two customers.

Local search is usually the bottleneck of most advanced metaheuristics for vehicle routing variants, and thus efficient evaluations of routes generated by the neighborhoods are critical for the overall algorithm's performance. When additional attributes (constraints, objectives or decisions) are considered, these route evaluations may be time consuming if implemented in a straightforward manner. 
To improve this process, \Gen\ relies on auxiliary data structures that collect partial information on any sub-sequence of consecutive customers in the incumbent solution.  This information is then used for efficiently evaluating the cost and feasibility of new routes generated by local search moves since any such move can be seen as a recombination of subsequences of consecutive customers from the incumbent solution.

For example, consider a route $r$ for a CVRP instance where customers $(1,2,3,4,5,6,7,8)$ are visited in the given order. To efficiently evaluate the capacity constraints, the partial load $Q(\sigma)$ for any sub-sequence $(\sigma)$ of the incumbent solution is preprocessed prior to move evaluations. An intra-route \textsc{Cross} move, of customers $(2,3)$ after $8$ requires evaluating route $r'= (1,4,5,6,7,8,2,3)$ with respect to cost and load feasibility. Loads $Q(\sigma_{1})$, $Q(\sigma_{23})$ and $Q(\sigma_{48})$ are known for sequences $(1)$, $(2,3)$ and $(4,5,6,7,8)$. Denoting $\oplus$ as the concatenation operator, we have $Q(R') = Q(\sigma_{1} \oplus \sigma_{48} \oplus \sigma_{23}) =  Q(\sigma_{1}) +  Q(\sigma_{48}) +  Q(\sigma_{23})$. This load constraint can thus be checked in $\mathcal{O}(1)$ operations. Otherwise, a straightforward approach sweeping through the new route and cumulating the demands would take a number of operations proportional to the number of 
customers, that is, $\mathcal{O}(n)$. This type of route evaluation is referred to as move evaluation \emph{by concatenation} in \cite{Vidal2012b}. The computational complexity to update the auxiliary information on subsequences is usually dominated by the complexity of evaluating moves.

\subsection{Application to the CluVRP}

Our application of \Gen\ to the CluVRP relies on two contributions: a route representation based on an ordered sequence of clusters to reduce the search space, and efficient route evaluation procedures using concatenations to evaluate the cost of a route assimilated to a sequence of clusters. These methodological elements can be easily integrated into the \Gen\ framework, and it was possible to use the original UHGS code with the sole addition of a new route-evaluation operator.

The method relies on the fact that in any cluster $V_k$, the cost $\hat{c}_{ij}$ of the best Hamiltonian path between customer $i \in V_k$ and customer $j \in V_k$ that services all other customers in $V_k \setminus \{i,j\}$ has been preprocessed \citep{BattarraEV2014}. Using this information, it is possible to obtain from a route represented as a sequence of clusters the best sequence of visits to customers in polynomial time by solving the shortest path problem in an auxiliary graph $\mathcal{G}' = (\mathcal{V}',\mathcal{A}')$, as illustrated in Figure \ref{ShortestPath}.

\begin{figure}[h]
\centering
\hspace*{-1cm}
\includegraphics[width=18cm]{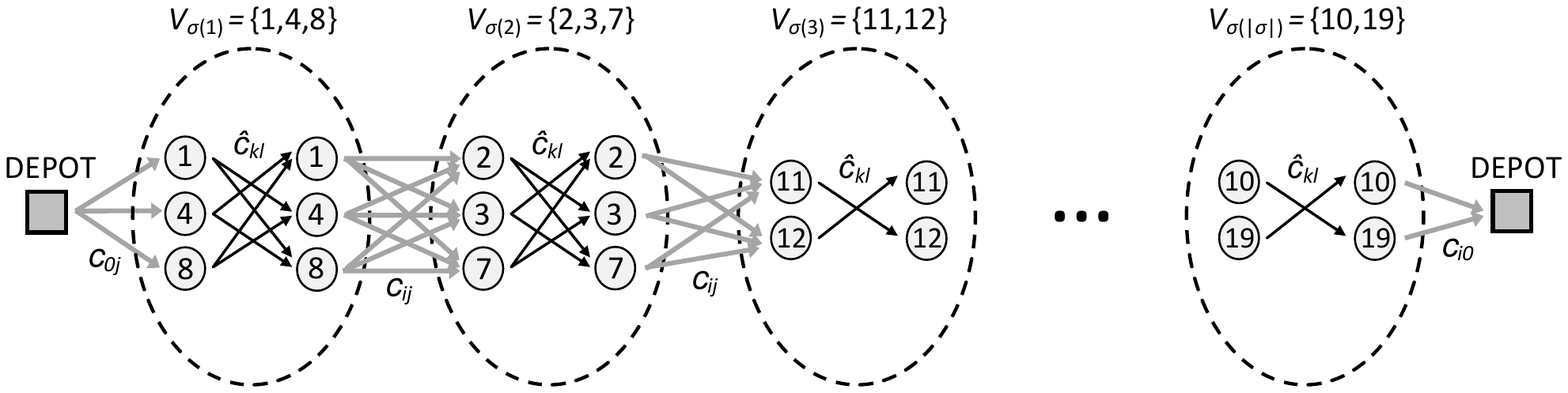}
\caption{Route representation in \Gen}
\label{ShortestPath}
\end{figure}

In Figure \ref{ShortestPath}, black lines correspond to precomputed Hamiltonian paths within clusters. For each cluster in the route, a set containing two copies of each node is generated. Pairs of node copies are connected by an arc and the cost of an arc $(k,l)$ is set to be the cost of the shortest Hamiltonian path $\hat{c}_{kl}$ in the cluster between the endpoints of the arcs. The depot is then connected to the first copy of each node in the first cluster $V_{\sigma(1)}$ by an arc $c_{0j}$, and the second copies of the nodes are connected to the first node copies of the next cluster, and so on. The cost associated to gray arcs is the travel distance between the endpoints.
A similar route representation was previously used for the GVRP by \cite{PopMS2013} and \cite{Vidal2013}. It leads to an implicit structural problem decomposition, considering only a VRP of a size proportional to the number of clusters  $N < n$. Difficult combinatorial decisions on path selections within clusters are thus relegated to the route-evaluation operators.

A straightforward application of this technique leads to route evaluations in $\mathcal{O}(N B^2)$ operations, where $B$ is the maximum number of customers in a cluster. These evaluations are computationally expensive. A contribution of this work is to show that efficient procedures based on preprocessing and concatenations allow for performing each move evaluation in amortized $O(B^2)$ operations, thus only depending on the square of the cluster size.
Our method preprocesses for each subsequence $\sigma = (\sigma(1),\dots,\sigma(|\sigma|))$ the shortest paths $S(\sigma)[i,j]$ that starts with any $i$\textsuperscript{th} customer of $\sigma$ and terminates at any $j$\textsuperscript{th} customer. The size of cluster $i$ is denoted as $\lambda_i$.

For a sequence $\sigma_0 = (s_k)$ containing a single cluster, if the cluster is restricted to a single customer $v_i$, then $S(\sigma_0)[i,i] = 0$, else $S(\sigma_0)[i,j] = +\infty$ for $i = j$ and $S(\sigma_0)[i,j] = \hat{c}_{ij}$ for $(i,j) \in \{1,\dots,\lambda_k\}^2$, $\hat{c}_{ij}$ being the distance of the best Hamiltonian path connecting $i$ and $j$ in the cluster.
As in \cite{Vidal2013}, the following equation enables us to evaluate $S(\sigma)$ on larger sub-sequences by induction on the concatenation operation. Note that it is a direct application of the Floyd-Warshall algorithm:

\begin{equation}
\begin{aligned}
S(\sigma_1 \oplus \sigma_2)[i,j] = \min_{1 \leq x \leq \lambda_{\sigma_1(|\sigma_1|)}, 1 \leq y \leq \lambda_{\sigma_2(1)}}  S(\sigma_1)[i,x] + c_{x y} + S(\sigma_2)[y,j]. \\
\forall i \in \{1,\dots,\lambda_{\sigma_1(1)}\}, \forall j  \in \{1,\dots,\lambda_{\sigma_2(|\sigma_2|)}\} \label{merge-general}
\end{aligned}
\end{equation}

 Equation (\ref{merge-general}) can therefore be used to perform preprocessing on all subsequences of customers. The same equation is then used during move evaluations to compute the cost of a new route as a concatenation of a  bounded number of existing subsequences with limited effort. Indeed, as illustrated in Figure \ref{ShortestPath2}, preprocessing this data is equivalent to preprocessing all-pairs of shortest paths  between nodes in each subsequence (in boldface in the figure). As a consequence, the size of the shortest path graph considered during separate move evaluations is considerably reduced.

\begin{figure}[h]
\centering
\hspace*{-1cm}
\includegraphics[width=18cm]{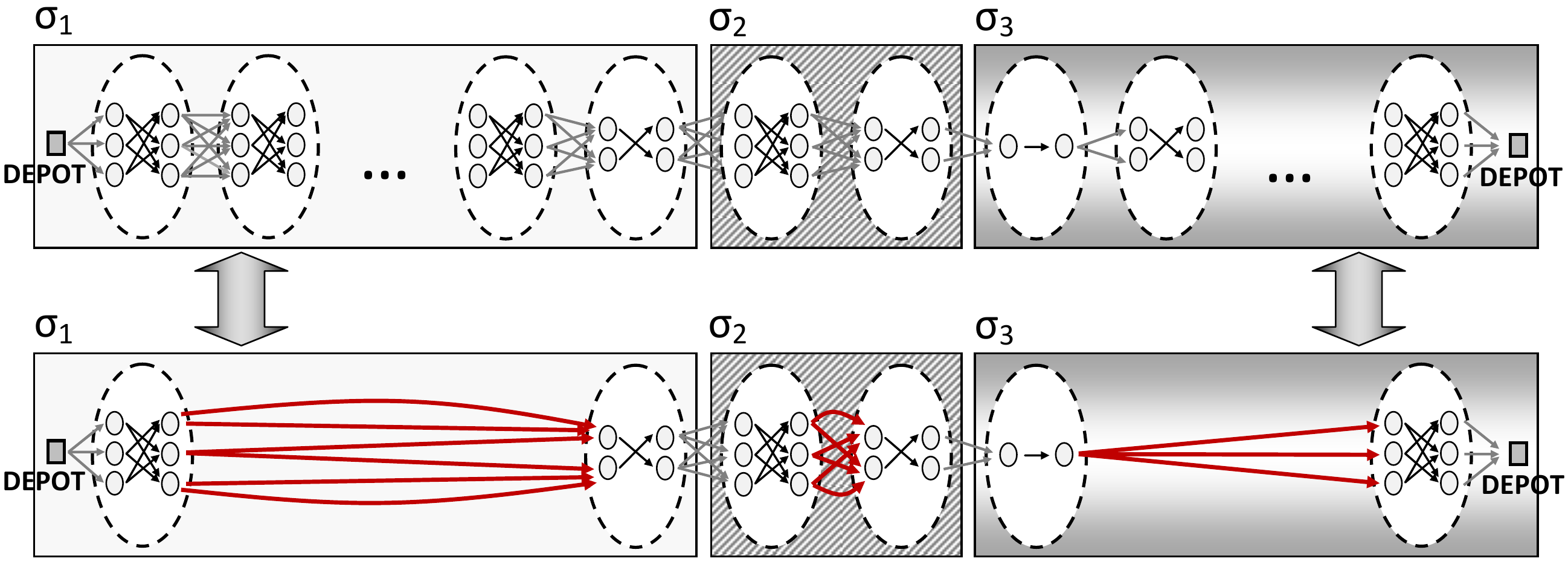}
\caption{Using preprocessed information on subsequences}
\label{ShortestPath2}
\end{figure}

\begin{proposition}
Using the proposed preprocessing, the amortized complexity of move evaluations, for classic VRP neighborhoods such as \textsc{Relocate}, \textsc{Swap}, \textsc{2-Opt}, \textsc{2-Opt*}, is $\mathcal{O}(B^2)$ instead of $\mathcal{O}(N B^2)$.
\end{proposition}

\begin{proof}
First, from the current incumbent solution, the preprocessing phase requires computing the shortest paths between each pair of nodes, for each route. For each route, the graph  $\mathcal{G}'$ is directed and acyclic. Equation (\ref{merge-general}) is applied iteratively, in lexicographic order starting from any cluster $\sigma_i$, $i \in \{1,\dots,|\sigma|\}$ and iteratively applied to $\sigma_j$ for $j \in \{i+1,\dots,|\sigma|\}$ to produce all shortest paths. This equation is thus used $\mathcal{O}(N^2)$ times to perform a complete preprocessing on all routes. Each evaluation of this expression requires $\mathcal{O}(B^2)$ time. The total effort for the preprocessing phase is  $\mathcal{O}(N^2 B^2)$.

After preprocessing, a local search using classic VRP neighborhoods is performed. Any move based on less than $k$ edge exchanges can be assimilated to a recombination of up to $k+1$ subsequences of consecutive clusters. This is the case for the mentioned neighborhoods with $k \leq 4$. Thus, each  move evaluation is performed with a bounded number of calls to Equation (\ref{merge-general}), in $\mathcal{O}(B^2)$ elementary operations.
The size $S$ of each neighborhood is quadratic in the number of clusters (e.g. swapping any cluster $i$ with cluster $j$ leads to $S = \Theta(N^2)$ possible moves), such that a complete neighborhood exploration takes $\mathcal{O}(B^2 N^2)$ time. The amortized complexity, for each move evaluation, considering both preprocessing and effective evaluation is thus $\mathcal{O}( \frac{N^2 B^2}{S}) = \mathcal{O}(B^2)$.
\end{proof}

\section{Computational results}\label{Sec.CompResults}

Computational experiments have been conducted on multiple benchmark instance sets. The first sets have been recently presented in \cite{BattarraEV2014}. The authors considered instances proposed for the GVRP literature by \cite{Bektas11} (namely GVRP2 and GVRP3 sets) and then generated larger instances by adapting the CVRP instances proposed by \cite{GoldenWKC1998} (namely Golden) with the same method as for the instance proposed by \cite{Bektas11}. Among the instances proposed in \cite{BattarraEV2014}, we include in our benchmark set all Golden instances (200 up to 484 customers) and the most challenging ones among the GVRP2 and GVRP3 sets (the instances denoted as G and C in the GVRP literature, with 101 up to 262 customers). 
We also generated an additional instance set with even larger problems (called hereafter Li), by adapting with the same logic as in \cite{Bektas11} the instances originally proposed by \cite{LiGW2005}. The latter set contains instances with up to 1200 customers. 
We decided to have clusters with average cardinality $\theta=5$, leading to larger instances with 121 up to 225 clusters.  
A summary of the characteristics of our benchmark set is provided in Table \ref{TabSummaryInst}. All sets of instances are available upon request and detailed result tables are displayed in the appendix.

\begin{table}[h]\center \caption{Summary of benchmark set characteristics}\label{TabSummaryInst}
\setlength{\extrarowheight}{1.5pt}
 \begin{tabular}{l|c|c|c|c}                                                                                                                                                                                                                                                      Instance Set	&	Source	&	\# Inst.& $n$&$|\mathcal{C}|$			\\
\hline											
C	&	\cite{Bektas11}	&	2 & 101-200&	34-100	\\
G	&	\cite{Bektas11}	&	8 & 262-262	&	88-131	\\
Golden	&	\cite{BattarraEV2014}&220&	201-481	&	17-97	\\
Li	&	New	&	12 & 560-1200	&	113	-241	\\
\hline											
\end{tabular}											

 \end{table} 

An extensive calibration effort was spent in previous literature to find good and robust parameters for UHGS \citep{Vidal2012} and ILS \citep{Subramanian2012}. We have relied on this knowledge to obtain an initial parameter setting, and then scaled the parameters controlling algorithm termination to generate solutions for large-scale instances in reasonable CPU time. As such, the population-size parameters of UHGS are set to $(\mu^\textsc{min},\mu^\textsc{gen}) = (8,8)$ and the termination criterion is $It_{max} = 400$. For ILS, the number of restarts has been set to $n_R =  50$ and the number of shaking iteration is $n_I = n + 5m$ as in \citet{Subramanian2012}. The choice of $n_I = 1000$ was adopted for \ILSClu.
All experiments have been conducted on a Xeon CPU with 3.07 GHz and 16 GB of RAM, running under Oracle Linux Server 6.4. Each algorithm was executed 10 times for each instance using a different random seed.


Table \ref{TabAllGCLi} summarizes the results obtained by the proposed \mycomment{hybrid} metaheuristics.  For each benchmark set, the number of instances ``Inst'' is given, as well as the number of times ``\# BKS'' the best known solution is found by \ILS, \ILSClu\ and \Gen, respectively. Columns  6-9 provide the average CPU time per instance in seconds. $UHGS_{p}$ also includes the CPU time dedicated to computing the cost of all intra-cluster Hamiltonian paths with Concorde. Columns 10-12 report the average percentage of deviation from the best known solutions ``Avg. \% Dev.''. Note that the percentage deviation for a solution of value $z$ from the best known solution value $z_{BKS}$ is computed as $\frac{z-z_{BKS}}{z_{BKS}}\times 100$. The last row reports the overall number of best known solutions found by each method, the average CPU time and percentage average deviation. 

From the experiments, it appears that \Gen\ is capable of finding most of the best known solutions (234 out of 242). In most cases, the average percentage gaps among the three methods is still small: \ILSClu\ has an average deviation of 0.19\% from the best known solutions and \ILS\ has an average deviation of 0.13\%.
\ILS\ is remarkably slower than the two other algorithms. The average CPU time for the large instances in the Li data set is 9548.6 seconds, versus 535.8, 345.3, 660.0 seconds of \ILSClu, \Gen, and \Gen$_p$, respectively. Despite the simplicity of adapting a CVRP metaheuristic to the \Our\ by including a penalization $M$ term, this resulting algorithm is much slower than the algorithms that take full advantage of the cluster constraints. Note that \ILS\ performs about 0.06\% better than \ILSClu, but \ILSClu\ is 15 times faster on average. 

\Gen$_p$ is faster on average than \ILS, even with the exhaustive search of all intra-cluster Hamiltonian paths using Concorde. This preprocessing phase is fast when the average cluster size is limited, but requires large CPU time when the cluster size increases, as in the case of the Golden instances. A heuristic evaluation of the cost of intra-cluster Hamiltonian paths could be a viable alternative. This is left as a research perspective.  Finally, \Gen\ is faster than \ILSClu\ for very large instances. \ILS-Clu\ is on average faster on the G and C data sets, but slower on average on the Li set. 

\begin{table}[h]\center \caption{Summary of results for the G, C, Golden and Li data set}\label{TabAllGCLi} \small
\hspace*{-0.7cm}
\setlength{\extrarowheight}{1pt}
\scalebox{0.9}
{
 \begin{tabular}{l|c|ccc|cccc|ccc}                                                                                                                                                                                                                                                             																								
		&		&	\multicolumn{3}{c}{\# BKS}					&	\multicolumn{4}{c}{Avg. Time (s)}							&	\multicolumn{3}{c}{Avg. \% Dev.}					\\
	Instance Set	&	$|Inst.|$	&	\ILS\	&	\ILSClu\	&	\Gen\	&	\ILS\	&	\ILSClu\	&	\Gen\	&	\Gen$_p$	&	\ILS\	&	\ILSClu\	&	\Gen\		\\
	\hline																							
	G	&	2	&	0	&	1	&	2	&	127.6	&	53.5	&	150.2	&	165.2	&	0.64	&	0.22	&	0.00	\\
	C	&	8	&	6	&	8	&	7	&	26.0	&	17.8	&	27.1	&	35.1	&	0.19	&	0.04	&	0.05	\\
	Golden&	220	&	127	&	87	&	213	&	698.8	&	53.9	&	53.7	&	854.9   &   0.11    &	0.19	&	0.01	\\
	Li	&	12	&	1	&	0	&	12	&	9548.6	&	535.8	&	345.3	&	660.0	&	0.34    &	0.21	&	0.00	\\
	\hline																							
	Tot:	&	242	&	134	&	96	&	234	&	1110.7	&	76.6	&	68.1	&	812.4	&	0.13	&	0.19	&	0.01	\\
	\end{tabular}																							

}
 \end{table}  

\begin{table}[h]\center \caption{Summary of results for the Golden data set grouped by instance}\label{TabSummaryGoldenI}
\hspace*{-1.2cm}
\setlength{\extrarowheight}{1.5pt}
\scalebox{0.95}
{
 \begin{tabular}{l|c|ccc|cccc|ccc}	&	&	\multicolumn{3}{c}{\#	Opt.}	&	\multicolumn{4}{c}{Avg.	Time	(s)}	&	\multicolumn{3}{c}{Avg.	\%	Dev.}	\\												
	Golden	&	$n$	&	\ILS\	&	\ILSClu\	&	\Gen\	&	\ILS\	&	\ILSClu\	&	\Gen\	& \Gen\ $_p$ &	\ILS\	&	\ILSClu\	&	\Gen\	\\			
	\hline																								
	1	&	241	&	11	&	9	&	11	&	141.92	&	23.95	&	22.4	&	172.94	&		0	&	0.03	&	0	\\
	2	&	321	&	3	&	2	&	11	&	442.25	&	47.27	&	47.57	&	243.39	&		0.08	&	0.07	&	0	\\
	3	&	401	&	5	&	2	&	11	&	1115.99	&	85.16	&	82.91	&	1384.18	&		0.12	&	0.14	&	0	\\
	4	&	481	&	1	&	0	&	9	&	2336.64	&	130.52	&	137.95	&	2608.13	&		0.12	&	0.13	&	0.01	\\
	5	&	201	&	11	&	11	&	11	&	81.18	&	19.92	&	14.57	&	2866.57	&		0	&	0	&	0	\\
	6	&	281	&	10	&	8	&	11	&	308.71	&	47.15	&	34.14	&	3848.32	&		0	&	0.03	&	0	\\
	7	&	361	&	3	&	1	&	11	&	816.36	&	73.97	&	66.94	&	2220.30	&		0.09	&	0.13	&	0	\\
	8	&	441	&	4	&	0	&	10	&	1573.62	&	101.69	&	97.09	&	1017.63	&		0.08	&	0.16	&	0	\\
	9	&	256	&	10	&	9	&	10	&	148.03	&	21.57	&	22.09	&	135.45	&		0.03	&	0.06	&	0.03	\\
	10	&	324	&	6	&	4	&	11	&	336.87	&	31.55	&	43.82	&	175.73	&		0.19	&	0.31	&	0	\\
	11	&	400	&	3	&	0	&	11	&	658.46	&	46.9	&	59.62	&	198.00	&		0.3	&	0.56	&	0	\\
	12	&	484	&	3	&	1	&	11	&	1420.29	&	72.22	&	94.05	&	389.16	&		0.48	&	0.73	&	0	\\
	13	&	253	&	8	&	5	&	11	&	145.67	&	21.86	&	23.33	&	164.69	&		0.05	&	0.12	&	0	\\
	14	&	321	&	7	&	2	&	11	&	333.56	&	34.29	&	38.12	&	152.78	&		0.12	&	0.25	&	0	\\
	15	&	397	&	1	&	0	&	9	&	713.35	&	52.3	&	65.99	&	279.15	&		0.33	&	0.48	&	0.03	\\
	16	&	481	&	3	&	3	&	11	&	1344.21	&	74.34	&	84.3	&	246.31	&		0.15	&	0.39	&	0	\\
	17	&	241	&	11	&	11	&	11	&	159.43	&	27.2	&	20.92	&	176.87	&		0	&	0	&	0	\\
	18	&	301	&	10	&	11	&	11	&	309.59	&	37.9	&	30.75	&	191.26	&		0.02	&	0	&	0	\\
	19	&	361	&	10	&	7	&	11	&	523.32	&	52.43	&	38.65	&	276.27	&		0.01	&	0.04	&	0	\\
	20	&	421	&	7	&	1	&	10	&	886.08	&	77.59	&	49.59	&	351.74	&		0.11	&	0.22	&	0.09	\\
	\hline																								
	Tot:	&	&	127	&	87	&	213	&	&	&	&	&	&  &	\\										
	\end{tabular}

}
 \end{table}  

 \begin{table}[h]\center \caption{Summary of results for the Golden data set grouped by average cluster size}\label{TabSummaryGolden}
\hspace*{-0.7cm}
\setlength{\extrarowheight}{1.5pt}
\scalebox{0.95}
{
 \begin{tabular}{l|ccc|cccc|ccc}
  &	\multicolumn{3}{c}{\# Opt.}	&	\multicolumn{4}{c}{Avg. Time (s)}	&	\multicolumn{3}{c}{Avg. \   Dev.}	\\
	$\theta$	&	\ILS\	&	\ILSClu\	&	\Gen\	&	\ILS\	&	\ILSClu\	&	\Gen$_p$\	&	\Gen$_p$	&	\ILS\	&	\ILSClu\	&	\Gen\	\\
\hline
	5	&	17	&	8	&	20	&	670.40	&	55.60	&	36.29	&	140.32	&	0.02  	&	0.11  	&	0.00  	\\
	6	&	18	&	12	&	20	&	663.38	&	53.84	&	37.22	&	155.89	&	0.02  	&	0.08  	&	0.00  	\\
	7	&	17	&	9	&	20	&	670.84	&	52.68	&	39.55	&	173.19	&	0.01  	&	0.11  	&	0.00  	\\
	8	&	11	&	9	&	20	&	688.00	&	50.83	&	43.15	&	251.55	&	0.08  	&	0.12  	&	0.00  	\\
	9	&	12	&	10	&	20	&	689.86	&	48.98	&	45.71	&	307.15	&	0.08  	&	0.18  	&	0.00  	\\
	10	&	12	&	7	&	20	&	691.00	&	49.16	&	49.64	&	553.37	&	0.11  	&	0.18  	&	0.00  	\\
	11	&	10	&	9	&	20	&	709.81	&	48.85	&	50.82	&	417.97	&	0.13  	&	0.18  	&	0.00  	\\
	12	&	9	&	9	&	20	&	695.70	&	50.12	&	54.77	&	1025.66	&	0.13  	&	0.18  	&	0.00  	\\
	13	&	8	&	5	&	20	&	725.73	&	53.12	&	69.04	&	916.16	&	0.18  	&	0.29  	&	0.00  	\\
	14	&	7	&	5	&	17	&	699.89	&	59.99	&	73.52	&	2327.81	&	0.24  	&	0.37  	&	0.06  	\\
	15	&	6	&	4	&	16	&	682.94	&	70.72	&	91.43	&	3135.31	&	0.23  	&	0.31  	&	0.03  	\\
	\hline
	Avg:	&	11.55	&	7.91	&	19.36	&	689.78	&	53.99	&	53.74	&	854.94	&	0.11 	&	0.19 	&	0.01 	\\
\end{tabular}																						

}
 \end{table}  

A more detailed comparison of the algorithms is displayed in Table \ref{TabSummaryGoldenI} for the Golden data set. The large number of instances in this set allows for an analysis of the algorithms' performances by varying number of the customers and cluster size. The table reports aggregated results, obtained by averaging over instances with the same number of customers. 
 
A correlation between the size of the instance and the performance of \ILS\ can be observed; larger instances lead to larger gaps and higher CPU time.  On the other hand, the performance of \ILSClu\ is less dependent on instance size. For example, instances of group 12 with 484 customers are the most challenging for \ILSClu\ with a 0.73\% average deviation, but the deviation for instances of group 4, with size $n = 481$, is only 0.13\% in average. A similar observation stands for \Gen, the most challenging instance groups being 4, 9, 14, and 20 with 481, 256, 397 and 421 customers, respectively.
 
 Aggregating the Golden instances by average cluster size $\theta$, as done in Table \ref{TabSummaryGolden}, leads to a further level of understanding of algorithms performance. All algorithms find solutions close to the best known when the average cluster size is large and therefore less clusters are present. The average CPU time of \ILS\ does not depend on the average cluster size, whereas \Gen\ is consistently faster when large and few clusters are present. \ILSClu\ attains its minimum CPU time when the average cluster size is approximately 9 customers. This is due to the fact that \ILSClu\ performs both intra and inter-cluster LS moves; a balanced instance in terms of number and size of the clusters is a good compromise in terms of CPU time. Finally \Gen\ was capable of improving the best known solutions for five instances from \cite{BattarraEV2014}. The values of these solutions are listed in Table \ref{TabAll3}.

\section{Conclusions} \label{SecConclusions}

This paper focused on the \Our, a generalization of the CVRP where customers are grouped into clusters. 
Three metaheuristics have been proposed, two of which are based on iterated local search, while the third is a hybrid genetic algorithm with a cluster-based solution representation. Efficient large neighborhood search procedures based on re-optimization techniques have been developed and integrated with the hybrid genetic search. The resulting three methods produce high quality solutions, and algorithms taking advantage of the cluster structure \mycomment{ and large neighborhoods} are remarkably faster. The hybrid genetic algorithm and large neighborhood search leads to solutions of higher quality that the two ILS based algorithms, but its pre-processing phase may become time consuming for instances with large clusters. Future work should consider heuristic preprocessing techniques to enhance CPU time, and other large neighborhoods strategies taking advantage of clusters.

\section{Acknowledgements}\label{Acknowledgements}
This work was partially supported by CORMSIS, Centre of Operational Research, Management Sciences and Information Systems, and by CNPq, Conselho Nacional de Desenvolvimento Cient{\'i}fico e Tecnol{\'o}gico (grant 471158/2012-7). 


\begin{thebibliography}{28}
\expandafter\ifx\csname natexlab\endcsname\relax\def\natexlab#1{#1}\fi
\expandafter\ifx\csname url\endcsname\relax
  \def\url#1{{\tt #1}}\fi
\expandafter\ifx\csname urlprefix\endcsname\relax\def\urlprefix{URL }\fi
\expandafter\ifx\csname urlstyle\endcsname\relax
  \expandafter\ifx\csname doi\endcsname\relax
  \def\doi#1{doi:\discretionary{}{}{}#1}\fi \else
  \expandafter\ifx\csname doi\endcsname\relax
  \def\doi{doi:\discretionary{}{}{}\begingroup \urlstyle{rm}\Url}\fi \fi

\bibitem[{Applegate et~al.(2001)Applegate, Bixby, Chv\`atal, and
  Cook}]{concorde}
Applegate, D., R.~Bixby, V.~Chv\`atal, W.~Cook. 2001.
\newblock Concorde tsp solver.

\bibitem[{Barth\'elemy et~al.(2010)Barth\'elemy, Rossi, Sevaux, and
  S\"{o}rensen}]{Barthelemy10}
Barth\'elemy, T., A.~Rossi, M.~Sevaux, K.~S\"{o}rensen. 2010.
\newblock Metaheuristic approach for the clustered vrp.
\newblock {\it EU/ME 2010 - 10th anniversary of the metaheuristic community\/}.
  Lorient, France.

\bibitem[{Battarra et~al.(2014)Battarra, Erdo\u{g}an, and
  Vigo}]{BattarraEV2014}
Battarra, M., G.~Erdo\u{g}an, D.~Vigo. 2014.
\newblock The clustered vehicle routing problem.
\newblock {\it Operations Research\/} {\bf 62} 58--71.

\bibitem[{Bekta\c{s} et~al.(2011)Bekta\c{s}, Erdo\u{g}an, and Ropke}]{Bektas11}
Bekta\c{s}, T., G.~Erdo\u{g}an, S.~Ropke. 2011.
\newblock Formulations and branch-and-cut algorithms for the generalized
  vehicle routing problem.
\newblock {\it Transportation Science\/} {\bf 45} 299--316.

\bibitem[{Clarke and Wright(1964)}]{ClarkeW1979}
Clarke, G., J.~W. Wright. 1964.
\newblock {Scheduling of Vehicles from a Central Depot to a Number of Delivery
  Points}.
\newblock {\it Operations research\/} {\bf 12} 568--581.

\bibitem[{Ghiani and Improta(2000)}]{GhianiI2000}
Ghiani, G., G.~Improta. 2000.
\newblock An efficient transformation of the generalized vehicle routing
  problem.
\newblock {\it European Journal of Operational Research\/} {\bf 122} 11--17.

\bibitem[{Goel and Vidal(2013)}]{Goel2012}
Goel, A., T.~Vidal. 2013.
\newblock {Hours of service regulations in road freight transport: an
  optimization-based international assessment}.
\newblock {\it Transportation Science, Articles in Advance\/} .

\bibitem[{Golden et~al.(1998)Golden, Wasil, Kelly, and Chao}]{GoldenWKC1998}
Golden, B.L., E.A. Wasil, J.~P. Kelly, I.-M. Chao. 1998.
\newblock Metaheuristics in vehicle routing.
\newblock {\it Fleet Management and Logistics\/}. Kluwer, Boston, 33--56.

\bibitem[{Li et~al.(2005)Li, Golden, and Wasil}]{LiGW2005}
Li, F., G.~Golden, E.~Wasil. 2005.
\newblock Very large-scale vehicle routing: New test problems, algorithms, and
  results.
\newblock {\it Computers \& Operations Research\/} {\bf 32} 1165--1179.

\bibitem[{Penna et~al.(2013)Penna, Subramanian, and Ochi}]{PennaSO2013}
Penna, P.H.V., A.~Subramanian, L.S. Ochi. 2013.
\newblock An iterated local search heuristic for the heterogeneous fleet
  vehicle routing problem.
\newblock {\it Journal of Heuristics\/} {\bf 19} 201--232.

\bibitem[{Pisinger and Ropke(2007)}]{PisingerR2007}
Pisinger, D., S.~Ropke. 2007.
\newblock A general heuristic for vehicle routing problems.
\newblock {\it Computers \& Operations Research\/} {\bf 34} 2403--2435.

\bibitem[{Pop et~al.(2012)Pop, Kara, and Horvat-Marc}]{PopKH2012}
Pop, P.C., I.~Kara, A.~Horvat-Marc. 2012.
\newblock New mathematical models of the generalized vehicle routing problem
  and extensions.
\newblock {\it Applied Mathematical Modelling\/} {\bf 36} 97--107.

\bibitem[{Pop et~al.(2013)Pop, Matei, and Pop~Sitar}]{PopMS2013}
Pop, P.C., O.~Matei, C.~Pop~Sitar. 2013.
\newblock An improved hybrid algorithm for solving the generalized vehicle
  routing problem.
\newblock {\it Neurocomputing\/} {\bf 109} 76--83.

\bibitem[{Prins(2004)}]{Prins2004}
Prins, C. 2004.
\newblock {A simple and effective evolutionary algorithm for the vehicle
  routing problem}.
\newblock {\it Computers \& Operations Research\/} {\bf 31} 1985--2002.

\bibitem[{Sevaux and S\"{o}rensen(2008)}]{SevauxS2008}
Sevaux, M., K.~S\"{o}rensen. 2008.
\newblock Hamiltonian paths in large clustered routing problems.
\newblock {\it Proceedings of the EU/MEeting 2008 workshop on Metaheuristics
  for Logistics and Vehicle Routing\/}. Troyes, France.

\bibitem[{Silva et~al.(2012)Silva, Subramanian, Vidal, and Ochi}]{Silva2012}
Silva, M.M., A.~Subramanian, T.~Vidal, L.S. Ochi. 2012.
\newblock {A simple and effective metaheuristic for the Minimum Latency
  Problem}.
\newblock {\it European Journal of Operational Research\/} {\bf 221} 513--520.

\bibitem[{S\"{o}rensen et~al.(2008)S\"{o}rensen, {Van den Bergh}, Cattrysse,
  and Sevaux}]{Sorensen08}
S\"{o}rensen, K., J.~{Van den Bergh}, D.~Cattrysse, M.~Sevaux. 2008.
\newblock A multiobjective distributionproblem for parcel delivery at tnt.
\newblock {\it Invited Talk at the International Workshop on Vehicle Routing in
  Practice, VIP'08\/}. Oslo, Norway.

\bibitem[{Subramanian(2012)}]{Subramanian2012}
Subramanian, A. 2012.
\newblock Heuristic, exact and hybrid approaches for vehicle routing problems.
\newblock Ph.D. thesis, Universitade Federal Fluminense.

\bibitem[{Subramanian and Battarra(2013)}]{SubramanianB2013}
Subramanian, A., M.~Battarra. 2013.
\newblock An iterated local search algorithm for the traveling salesman problem
  with pickups and deliveries.
\newblock {\it Journal of the Operational Research Society\/} {\bf 64}
  402--409.

\bibitem[{Subramanian et~al.(2014)Subramanian, Battarra, and
  Potts}]{SubramanianBP2014}
Subramanian, A., M.~Battarra, C.~Potts. 2014.
\newblock An iterated local search heuristic for the single machine total
  weighted tardiness problem with sequence-dependent setup times.
\newblock {\it International Journal of Production Research\/} {\bf 52}
  2729--2742.

\bibitem[{Subramanian et~al.(2010)Subramanian, Drummond, Bentes, Ochi, and
  Farias}]{SubramanianDBOF2010}
Subramanian, A., L.M.A. Drummond, C.~Bentes, L.S. Ochi, R.~Farias. 2010.
\newblock A parallel heuristic for the vehicle routing problem with
  simultaneous pickup and delivery.
\newblock {\it Computers \& Operations Research\/} {\bf 37} 1899--1911.

\bibitem[{Vidal(2013)}]{Vidal2013}
Vidal, T. 2013.
\newblock Approches g\'{e}n\'{e}rales de r\'{e}solution pour les probl\'{e}mes
  multi-attributs de tourn\'{e}s de v\'{e}hicules et confection d'horaires.
\newblock Ph.D. thesis, Université de Montr\'{e}al \& Universit\'{e} de
  Technologie de Troyes.

\bibitem[{Vidal et~al.(2012)Vidal, Crainic, Gendreau, Lahrichi, and
  Rei}]{Vidal2012}
Vidal, T., T.G. Crainic, M.~Gendreau, N.~Lahrichi, W.~Rei. 2012.
\newblock {A hybrid genetic algorithm for multidepot and periodic vehicle
  routing problems}.
\newblock {\it Operations Research\/} {\bf 60} 611--624.

\bibitem[{Vidal et~al.(2013{\natexlab{a}})Vidal, Crainic, Gendreau, and
  Prins}]{Vidal2012c}
Vidal, T., T.G. Crainic, M.~Gendreau, C.~Prins. 2013{\natexlab{a}}.
\newblock {A hybrid genetic algorithm with adaptive diversity management for a
  large class of vehicle routing problems with time-windows}.
\newblock {\it Computers \& Operations Research\/} {\bf 40} 475--489.

\bibitem[{Vidal et~al.(2013{\natexlab{b}})Vidal, Crainic, Gendreau, and
  Prins}]{Vidal2012a}
Vidal, T., T.G. Crainic, M.~Gendreau, C.~Prins. 2013{\natexlab{b}}.
\newblock {Heuristics for multi-attribute vehicle routing problems: a survey
  and synthesis}.
\newblock {\it European Journal of Operational Research\/} {\bf 231} 1--21.

\bibitem[{Vidal et~al.(2014{\natexlab{a}})Vidal, Crainic, Gendreau, and
  Prins}]{Vidal2012b}
Vidal, T., T.G. Crainic, M.~Gendreau, C.~Prins. 2014{\natexlab{a}}.
\newblock {A unified solution framework for multi-attribute vehicle routing
  problems}.
\newblock {\it European Journal of Operational Research\/} {\bf 234} 658--673.

\bibitem[{Vidal et~al.(2014{\natexlab{b}})Vidal, Crainic, Gendreau, and
  Prins}]{Vidal2012e}
Vidal, T., T.G. Crainic, M.~Gendreau, C.~Prins. 2014{\natexlab{b}}.
\newblock Implicit depot assignments and rotations in vehicle routing
  heuristics.
\newblock {\it European Journal of Operational Research\/} {\bf 237} 15--28.

\bibitem[{Vidal et~al.(2014{\natexlab{c}})Vidal, Maculan, Ochi, and
  Penna}]{Vidal2014}
Vidal, T., N.~Maculan, L.S. Ochi, P.H.V. Penna. 2014{\natexlab{c}}.
\newblock {Large neighborhoods with implicit customer selection for vehicle
  routing problems with profits}.
\newblock Tech. rep., CIRRELT, Montr\'{e}al.

\end{thebibliography}

\section{Detailed results}
\begin{landscape}

\begin{table}[!htb]\center \caption{Detailed results, Golden 1--4.}\label{TabAll1} \tiny
\setlength{\extrarowheight}{1.5pt}
 \begin{tabular}{l|ccc|c|ccc|ccc|ccccc}																						
	&	&	&	&	Exact	&	\multicolumn{3}{c|}{\ILS}	&	\multicolumn{3}{c|}{\ILSClu}	&	\multicolumn{5}{c}{\Gen}\\			
Inst.	&	$n$	&	$|\mathcal{C}|$	&	$m$	&	&	Best	&	Avg.	&		Avg.		&		Best		&	Avg.	&		Avg.		&		Best		&	Avg.	&		Avg.		&		Preproc.		&	Total\\					
	&	&	&	&	&	&	&Time(s)	&	&	&	Time(s)		&		&		&	Time(s)&		(s)		&		Time(s)\\																		
\hline																																														
Golden	1	&	241	&	17	&	4	&	\bf{	4831	}	&	\bf{	4831	}	&	 	4831.4	 	&	129.6	&	\bf{	4831	}	&	 	4831.9	 	&	23.9	&	\bf{	4831	}	&	\bf{	4831	}	&	14.2	&	541	&	555.2	\\
Golden	1	&	241	&	18	&	4	&	\bf{	4847	}	&	\bf{	4847	}	&	 	4849.6	 	&	134.3	&	\bf{	4847	}	&	 	4848.4	 	&	21.8	&	\bf{	4847	}	&	\bf{	4847	}	&	17.1	&	151	&	168.1	\\
Golden	1	&	241	&	19	&	4	&	\bf{	4872	}	&	\bf{	4872	}	&	 	4877.8	 	&	137	&	\bf{	4872	}	&	 	4877.1	 	&	22	&	\bf{	4872	}	&	\bf{	4872	}	&	17.5	&	151	&	168.5	\\
Golden	1	&	241	&	21	&	4	&	\bf{	4889	}	&	\bf{	4889	}	&	 	4896.2	 	&	135.8	&	\bf{	4889	}	&	 	4891.2	 	&	21.2	&	\bf{	4889	}	&	\bf{	4889	}	&	17.6	&	147	&	164.6	\\
Golden	1	&	241	&	22	&	4	&	\bf{	4908	}	&	\bf{	4908	}	&	 	4908.7	 	&	147	&	\bf{	4908	}	&	 	4912.3	 	&	22.2	&	\bf{	4908	}	&	\bf{	4908	}	&	17.6	&	151	&	168.6	\\
Golden	1	&	241	&	25	&	4	&	\bf{	4899	}	&	\bf{	4899	}	&	 	4905.9	 	&	134.4	&	\bf{	4899	}	&	 	4899.9	 	&	24	&	\bf{	4899	}	&	\bf{	4899	}	&	20	&	144	&	164	\\
Golden	1	&	241	&	27	&	4	&	\bf{	4934	}	&	\bf{	4934	}	&	 	4943	 	&	143.1	&	\bf{	4934	}	&	 	4935.3	 	&	22.4	&	\bf{	4934	}	&	\bf{	4934	}	&	20.6	&	115	&	135.6	\\
Golden	1	&	241	&	31	&	4	&	\bf{	5050	}	&	\bf{	5050	}	&	 	5051.5	 	&	150	&	\bf{	5050	}	&	 	5050.2	 	&	23.6	&	\bf{	5050	}	&	\bf{	5050	}	&	21.6	&	90	&	111.6	\\
Golden	1	&	241	&	35	&	4	&	\bf{	5102	}	&	\bf{	5102	}	&	 	5113.3	 	&	152.6	&	 	5108	 	&	 	5118	 	&	24.3	&	\bf{	5102	}	&	\bf{	5102	}	&	28.5	&	65	&	93.5	\\
Golden	1	&	241	&	41	&	4	&	\bf{	5097	}	&	\bf{	5097	}	&	 	5108.4	 	&	149.1	&	 	5107	 	&	 	5113.3	 	&	27.8	&	\bf{	5097	}	&	\bf{	5097	}	&	31.9	&	55	&	86.9	\\
Golden	1	&	241	&	49	&	3	&	\bf{	5000	}	&	\bf{	5000	}	&	 	5011.2	 	&	148.3	&	\bf{	5000	}	&	 	5013.8	 	&	30.5	&	\bf{	5000	}	&	\bf{	5000	}	&	39.7	&	46	&	85.7	\\
\hline																																														
Golden	2	&	321	&	22	&	4	&	\bf{	7716	}	&	\bf{	7716	}	&	 	7719.5	 	&	461.8	&	 	7718	 	&	 	7722.2	 	&	51.6	&	\bf{	7716	}	&	\bf{	7716	}	&	31	&	368	&	399	\\
Golden	2	&	321	&	23	&	4	&	\bf{	7693	}	&	\bf{	7693	}	&	 	7698.7	 	&	444.5	&	\bf{	7693	}	&	 	7700.2	 	&	51.4	&	\bf{	7693	}	&	\bf{	7693	}	&	31.6	&	332	&	363.6	\\
Golden	2	&	321	&	25	&	4	&	\bf{	7668	}	&	\bf{	7668	}	&	 	7675	 	&	441.9	&	 	7669	 	&	 	7677	 	&	49.3	&	\bf{	7668	}	&	\bf{	7668	}	&	32.8	&	305	&	337.8	\\
Golden	2	&	321	&	27	&	4	&	\bf{	7638	}	&	 	7644	 	&	 	7649.6	 	&	460.3	&	\bf{	7638	}	&	 	7649.4	 	&	45.2	&	\bf{	7638	}	&	 	7638.2	 	&	34.3	&	262	&	296.3	\\
Golden	2	&	321	&	30	&	4	&	\bf{	7617	}	&	 	7619	 	&	 	7635.3	 	&	447.4	&	 	7629	 	&	 	7644	 	&	38.8	&	\bf{	7617	}	&	\bf{	7617	}	&	37.6	&	172	&	209.6	\\
Golden	2	&	321	&	33	&	4	&	\bf{	7640	}	&	 	7649	 	&	 	7660.3	 	&	452.8	&	 	7646	 	&	 	7658	 	&	38.8	&	\bf{	7640	}	&	 	7641.5	 	&	48.3	&	152	&	200.3	\\
Golden	2	&	321	&	36	&	4	&	\bf{	7643	}	&	 	7650	 	&	 	7656.5	 	&	439.9	&	 	7653	 	&	 	7660.4	 	&	39.4	&	\bf{	7643	}	&	 	7649.7	 	&	49.1	&	167	&	216.1	\\
Golden	2	&	321	&	41	&	4	&	\bf{	7738	}	&	 	7759	 	&	 	7771.5	 	&	410.5	&	 	7742	 	&	 	7756.7	 	&	43.5	&	\bf{	7738	}	&	\bf{	7738	}	&	47.6	&	125	&	172.6	\\
Golden	2	&	321	&	46	&	4	&	\bf{	7861	}	&	 	7874	 	&	 	7888	 	&	431.8	&	 	7871	 	&	 	7884.7	 	&	45.3	&	\bf{	7861	}	&	 	7863.8	 	&	55.6	&	99	&	154.6	\\
Golden	2	&	321	&	54	&	4	&	\bf{	7920	}	&	 	7925	 	&	 	7937.9	 	&	427.5	&	 	7926	 	&	 	7931.6	 	&	53.4	&	\bf{	7920	}	&	 	7920.4	 	&	73	&	96	&	169	\\
Golden	2	&	321	&	65	&	4	&	\bf{	7892	}	&	 	7898	 	&	 	7910.8	 	&	446.4	&	 	7900	 	&	 	7906.3	 	&	63.2	&	\bf{	7892	}	&	 	7893.4	 	&	82.4	&	76	&	158.4	\\
\hline																																														
Golden	3	&	401	&	27	&	4	&	\bf{	10540	}	&	\bf{	10540	}	&	 	10547.6	 	&	1077.6	&	 	10547	 	&	 	10572.5	 	&	88.3	&	\bf{	10540	}	&	\bf{	10540	}	&	51.2	&	9678	&	9729.2	\\
Golden	3	&	401	&	29	&	4	&	\bf{	10504	}	&	\bf{	10504	}	&	 	10513.5	 	&	1020.5	&	 	10509	 	&	 	10520.9	 	&	83.7	&	\bf{	10504	}	&	\bf{	10504	}	&	58.2	&	2696	&	2754.2	\\
Golden	3	&	401	&	31	&	4	&	\bf{	10486	}	&	\bf{	10486	}	&	 	10497.5	 	&	1103.9	&	 	10500	 	&	 	10511.5	 	&	78.4	&	\bf{	10486	}	&	\bf{	10486	}	&	56	&	326	&	382	\\
Golden	3	&	401	&	34	&	4	&	\bf{	10465	}	&	\bf{	10465	}	&	 	10492.7	 	&	1093	&	 	10483	 	&	 	10497.6	 	&	79.1	&	\bf{	10465	}	&	\bf{	10465	}	&	59.6	&	372	&	431.6	\\
Golden	3	&	401	&	37	&	4	&	\bf{	10482	}	&	\bf{	10482	}	&	 	10504.1	 	&	1096.4	&	\bf{	10482	}	&	 	10515.2	 	&	78.1	&	\bf{	10482	}	&	\bf{	10482	}	&	64.6	&	346	&	410.6	\\
Golden	3	&	401	&	41	&	4	&	\bf{	10501	}	&	 	10518	 	&	 	10533	 	&	1194.8	&	 	10523	 	&	 	10544.2	 	&	79.5	&	\bf{	10501	}	&	\bf{	10501	}	&	76.1	&	290	&	366.1	\\
Golden	3	&	401	&	45	&	4	&	\bf{	10485	}	&	 	10502	 	&	 	10526.1	 	&	1078	&	 	10495	 	&	 	10525.1	 	&	76.4	&	\bf{	10485	}	&	\bf{	10485	}	&	69.3	&	97	&	166.3	\\
Golden	3	&	401	&	51	&	4	&	\bf{	10583	}	&	 	10605	 	&	 	10637.3	 	&	1172.1	&	\bf{	10583	}	&	 	10639.6	 	&	78.8	&	\bf{	10583	}	&	\bf{	10583	}	&	83.6	&	153	&	236.6	\\
Golden	3	&	401	&	58	&	4	&	\bf{	10776	}	&	 	10789	 	&	 	10810.7	 	&	1157.1	&	 	10797	 	&	 	10821.5	 	&	84.1	&	\bf{	10776	}	&	 	10777.9	 	&	124.8	&	124	&	248.8	\\
Golden	3	&	401	&	67	&	4	&	\bf{	10797	}	&	 	10824	 	&	 	10843.2	 	&	1110.9	&	 	10821	 	&	 	10844.8	 	&	96.4	&	\bf{	10797	}	&	 	10801.2	 	&	121.5	&	123	&	244.5	\\
Golden	3	&	401	&	81	&	4	&	\bf{	10614	}	&	 	10654	 	&	 	10669.7	 	&	1171.6	&	 	10662	 	&	 	10682	 	&	113.8	&	\bf{	10614	}	&	 	10621.6	 	&	147.1	&	109	&	256.1	\\
\hline																																														
Golden	4	&	481	&	33	&	4	&	\bf{	13598	}	&	 	13606	 	&	 	13614.3	 	&	2397.9	&	 	13602	 	&	 	13628.5	 	&	130.6	&	\bf{	13598	}	&	\bf{	13598	}	&	87.4	&	7460	&	7547.4	\\
Golden	4	&	481	&	35	&	4	&	\bf{	13643	}	&	\bf{	13643	}	&	 	13688.4	 	&	2362.8	&	 	13678	 	&	 	13706.4	 	&	127.2	&	\bf{	13643	}	&	\bf{	13643	}	&	88.5	&	7434	&	7522.5	\\
Golden	4	&	481	&	37	&	4	&	\bf{	13520	}	&	 	13525	 	&	 	13557.9	 	&	2207.7	&	 	13535	 	&	 	13569.3	 	&	113.5	&	\bf{	13520	}	&	\bf{	13520	}	&	89.1	&	1161	&	1250.1	\\
Golden	4	&	481	&	41	&	4	&	\bf{	13460	}	&	 	13466	 	&	 	13494.9	 	&	2116.6	&	 	13468	 	&	 	13492.7	 	&	113.4	&	\bf{	13460	}	&	\bf{	13460	}	&	95.9	&	5438	&	5533.9	\\
Golden	4	&	481	&	44	&	4	&	\bf{	13568	}	&	 	13588	 	&	 	13604.9	 	&	2113.7	&	 	13591	 	&	 	13601.5	 	&	112.8	&	\bf{	13568	}	&	 	13570.5	 	&	121.1	&	1789	&	1910.1	\\
Golden	4	&	481	&	49	&	4	&	\bf{	13758	}	&	 	13784	 	&	 	13812.8	 	&	2262.5	&	 	13767	 	&	 	13794.8	 	&	121.8	&	\bf{	13758	}	&	\bf{	13758	}	&	117.9	&	1659	&	1776.9	\\
Golden	4	&	481	&	54	&	4	&	\bf{	13760	}	&	 	13772	 	&	 	13829.5	 	&	2368	&	 	13779	 	&	 	13816.6	 	&	123.7	&	\bf{	13760	}	&	\bf{	13760	}	&	121.2	&	113	&	234.2	\\
Golden	4	&	481	&	61	&	4	&	\bf{	13791	}	&	 	13810	 	&	 	13857.6	 	&	2294.6	&	 	13800	 	&	 	13846.6	 	&	123.5	&	\bf{	13791	}	&	 	13791.9	 	&	133.3	&	1635	&	1768.3	\\
Golden	4	&	481	&	69	&	4	&	\bf{	13966	}	&	 	14000	 	&	 	14034.8	 	&	2584.6	&	 	13979	 	&	 	14022.8	 	&	136.1	&	\bf{	13966	}	&	 	13967.2	 	&	195.6	&	198	&	393.6	\\
Golden	4	&	481	&	81	&	4	&	\bf{	13975	}	&	 	13987	 	&	 	14031.5	 	&	2710.1	&	 	14010	 	&	 	14031.6	 	&	149.3	&	 	13977	 	&	 	13980.2	 	&	174	&	152	&	326	\\
Golden	4	&	481	&	97	&	4	&	\bf{	13775	}	&	 	13821	 	&	 	13857.5	 	&	2284.7	&	 	13808	 	&	 	13849	 	&	183.8	&	 	13783	 	&	 	13792.2	 	&	293.5	&	133	&	426.5	\\
\hline																																														
\end{tabular}																																														

 \end{table}
\end{landscape} 

\begin{landscape} 
 \begin{table}[!htb]\center \caption{Detailed results, Golden 5--8.}\label{TabAll2} \tiny
\setlength{\extrarowheight}{1.5pt}
 \begin{tabular}{l|ccc|c|ccc|ccc|ccccc}																																														
	&	&	&	&	Exact	&	\multicolumn{3}{c|}{\ILS}	&	\multicolumn{3}{c|}{\ILSClu}	&	\multicolumn{5}{c}{\Gen}\\																																			
Inst.	&	$n$	&	$|\mathcal{C}|$	&	$m$	&	&	Best	&	Avg.	&		Avg.		&		Best		&	Avg.	&		Avg.		&		Best		&	Avg.	&		Avg.		&		Preproc.		&	Total\\					
	&	&	&	&	&	&	&Time(s)	&	&	&	Time(s)		&		&		&	Time(s)&		(s)		&		Time(s)\\																		
\hline																																														
Golden	5	&	201	&	14	&	4	&	\bf{	7622	}	&	\bf{	7622	}	&	 	7622.5	 	&	74.5	&	\bf{	7622	}	&	 	7640.8	 	&	22.5	&	\bf{	7622	}	&	\bf{	7622	}	&	8.3	&	7651	&	7659.3	\\
Golden	5	&	201	&	15	&	3	&	\bf{	7424	}	&	\bf{	7424	}	&	 	7443.6	 	&	94.2	&	\bf{	7424	}	&	 	7427.6	 	&	23.9&		\bf{	7424	}	&	\bf{	7424	}	&	12.2	&	7568	&	7580.2	\\
Golden	5	&	201	&	16	&	3	&	\bf{	7491	}	&	\bf{	7491	}	&	\bf{	7491	}	&	94.3	&	\bf{	7491	}	&	\bf{	7491	}	&	25.1	&	\bf{	7491	}	&	\bf{	7491	}	&	13.3	&	7758	&	7771.3	\\
Golden	5	&	201	&	17	&	3	&	\bf{	7434	}	&	\bf{	7434	}	&	 	7445.8	 	&	90.6	&	\bf{	7434	}	&	 	7439.9	 	&	21.5	&	\bf{	7434	}	&	\bf{	7434	}	&	13.3	&	6893	&	6906.3	\\
Golden	5	&	201	&	19	&	4	&	\bf{	7576	}	&	\bf{	7576	}	&	\bf{	7576	}	&	71.2	&	\bf{	7576	}	&	\bf{	7576	}	&	16.9	&	\bf{	7576	}	&	\bf{	7576	}	&	12.8	&	498	&	510.8	\\
Golden	5	&	201	&	21	&	4	&	\bf{	7596	}	&	\bf{	7596	}	&	 	7596.3	 	&	76.7	&	\bf{	7596	}	&	 	7596.3	 	&	16	&	\bf{	7596	}	&	\bf{	7596	}	&	12.8	&	209	&	221.8	\\
Golden	5	&	201	&	23	&	4	&	\bf{	7643	}	&	\bf{	7643	}	&	\bf{	7643	}	&	78.7	&	\bf{	7643	}	&	\bf{	7643	}	&	16.9	&	\bf{	7643	}	&	\bf{	7643	}	&	14.6	&	151	&	165.6	\\
Golden	5	&	201	&	26	&	4	&	\bf{	7560	}	&	\bf{	7560	}	&	\bf{	7560	}	&	78.3	&	\bf{	7560	}	&	\bf{	7560	}	&	18	&	\bf{	7560	}	&	\bf{	7560	}	&	14.9	&	221	&	235.9	\\
Golden	5	&	201	&	29	&	4	&	\bf{	7410	}	&	\bf{	7410	}	&	\bf{	7410	}	&	75.3	&	\bf{	7410	}	&	\bf{	7410	}	&	17.1	&	\bf{	7410	}	&	\bf{	7410	}	&	16.5	&	219	&	235.5	\\
Golden	5	&	201	&	34	&	4	&	\bf{	7429	}	&	\bf{	7429	}	&	\bf{	7429	}	&	78.6	&	\bf{	7429	}	&	 	7430.6	 	&	20.1	&	\bf{	7429	}	&	\bf{	7429	}	&	19.4	&	164	&	183.4	\\
Golden	5	&	201	&	41	&	4	&	\bf{	7241	}	&	\bf{	7241	}	&	 	7243	 	&	80.5	&	\bf{	7241	}	&	\bf{	7241	}	&	21	&	\bf{	7241	}	&	\bf{	7241	}	&	22.1	&	40	&	62.1	\\
\hline																																														
Golden	6	&	281	&	19	&	3	&	\bf{	8624	}	&	\bf{	8624	}	&	 	8629.2	 	&	353.4	&	\bf{	8624	}	&	\bf{	8624	}	&	61.5	&	\bf{	8624	}	&	\bf{	8624	}	&	29.1	&	18971	&	19000.1	\\
Golden	6	&	281	&	21	&	3	&	\bf{	8628	}	&	\bf{	8628	}	&	 	8629.7	 	&	371.1	&	\bf{	8628	}	&	 	8632.4	 	&	59.8	&	\bf{	8628	}	&	\bf{	8628	}	&	27.7	&	20337	&	20364.7	\\
Golden	6	&	281	&	22	&	3	&	\bf{	8646	}	&	\bf{	8646	}	&	 	8648	 	&	378.7	&	\bf{	8646	}	&	 	8650.1	 	&	59.6	&	\bf{	8646	}	&	\bf{	8646	}	&	29.9	&	789	&	818.9	\\
Golden	6	&	281	&	24	&	4	&	\bf{	8853	}	&	\bf{	8853	}	&	 	8868.5	 	&	281	&	\bf{	8853	}	&	 	8869.8	 	&	41	&	\bf{	8853	}	&	\bf{	8853	}	&	26	&	525	&	551	\\
Golden	6	&	281	&	26	&	4	&	\bf{	8910	}	&	\bf{	8910	}	&	 	8916.6	 	&	271.9	&	\bf{	8910	}	&	 	8921.1	 	&	39.2	&	\bf{	8910	}	&	\bf{	8910	}	&	28	&	465	&	493	\\
Golden	6	&	281	&	29	&	4	&	\bf{	8936	}	&	\bf{	8936	}	&	 	8962.8	 	&	302.7	&	 	8949	 	&	 	8963.7	 	&	40.5	&	\bf{	8936	}	&	\bf{	8936	}	&	34.3	&	257	&	291.3	\\
Golden	6	&	281	&	32	&	4	&	\bf{	8891	}	&	\bf{	8891	}	&	 	8905.8	 	&	282.7	&	\bf{	8891	}	&	 	8899.9	 	&	39	&	\bf{	8891	}	&	\bf{	8891	}	&	30.6	&	92	&	122.6	\\
Golden	6	&	281	&	36	&	4	&	\bf{	8969	}	&	\bf{	8969	}	&	 	8970.4	 	&	284.7	&	\bf{	8969	}	&	 	8971.6	 	&	38.8	&	\bf{	8969	}	&	\bf{	8969	}	&	32.3	&	160	&	192.3	\\
Golden	6	&	281	&	41	&	4	&	\bf{	9028	}	&	\bf{	9028	}	&	 	9039.8	 	&	284.6	&	 	9039	 	&	 	9043.2	 	&	40.4	&	\bf{	9028	}	&	\bf{	9028	}	&	43	&	150	&	193	\\
Golden	6	&	281	&	47	&	4	&	\bf{	8923	}	&	\bf{	8923	}	&	 	8924.2	 	&	282.8	&	\bf{	8923	}	&	\bf{	8923	}	&	46.3	&	\bf{	8923	}	&	\bf{	8923	}	&	39.6	&	134	&	173.6	\\
Golden	6	&	281	&	57	&	4	&	\bf{	9028	}	&	 	9031	 	&	 	9052.6	 	&	302.2	&	 	9031	 	&	 	9052.1	 	&	52.6	&	\bf{	9028	}	&	 	9028.7	 	&	54.9	&	76	&	130.9	\\
\hline																																														
Golden	7	&	361	&	25	&	3	&	\bf{	9904	}	&	 	9909	 	&	 	9931	 	&	1038.6	&	 	9922	 	&	 	9947.9	 	&	90.2	&	\bf{	9904	}	&	\bf{	9904	}	&	57.7	&	12354	&	12411.7	\\
Golden	7	&	361	&	26	&	3	&	\bf{	9888	}	&	 	9907	 	&	 	9914.5	 	&	951.4	&	 	9905	 	&	 	9924.4	 	&	84.4	&	\bf{	9888	}	&	\bf{	9888	}	&	57.4	&	3017	&	3074.4	\\
Golden	7	&	361	&	28	&	3	&	\bf{	9917	}	&	\bf{	9917	}	&	 	9944	 	&	946.7	&	 	9922	 	&	 	9949.9	 	&	80	&	\bf{	9917	}	&	\bf{	9917	}	&	61.7	&	3092	&	3153.7	\\
Golden	7	&	361	&	31	&	4	&	\bf{	10021	}	&	 	10035	 	&	 	10048.9	 	&	737.3	&	 	10025	 	&	 	10054.7	 	&	66.1	&	\bf{	10021	}	&	\bf{	10021	}	&	48.5	&	2626	&	2674.5	\\
Golden	7	&	361	&	33	&	4	&	\bf{	10029	}	&	 	10040	 	&	 	10057.2	 	&	707.6	&	 	10042	 	&	 	10059.8	 	&	62.6	&	\bf{	10029	}	&	\bf{	10029	}	&	51.8	&	980	&	1031.8	\\
Golden	7	&	361	&	37	&	4	&	\bf{	10131	}	&	 	10141	 	&	 	10150	 	&	723.5	&	 	10134	 	&	 	10148.1	 	&	65.7	&	\bf{	10131	}	&	\bf{	10131	}	&	62	&	1005	&	1067	\\
Golden	7	&	361	&	41	&	4	&	\bf{	10052	}	&	\bf{	10052	}	&	 	10071.3	 	&	736.5	&	\bf{	10052	}	&	 	10067.6	 	&	67.1	&	\bf{	10052	}	&	\bf{	10052	}	&	58.1	&	116	&	174.1	\\
Golden	7	&	361	&	46	&	4	&	\bf{	10080	}	&	\bf{	10080	}	&	 	10100	 	&	744.7	&	 	10094	 	&	 	10113.1	 	&	64.9	&	\bf{	10080	}	&	\bf{	10080	}	&	62.5	&	201	&	263.5	\\
Golden	7	&	361	&	52	&	4	&	\bf{	10095	}	&	 	10096	 	&	 	10117.6	 	&	804.7	&	 	10116	 	&	 	10130.6	 	&	65.3	&	\bf{	10095	}	&	\bf{	10095	}	&	83	&	108	&	191	\\
Golden	7	&	361	&	61	&	4	&	\bf{	10096	}	&	 	10115	 	&	 	10159.8	 	&	823.9	&	 	10128	 	&	 	10139.5	 	&	77.4	&	\bf{	10096	}	&	\bf{	10096	}	&	84.5	&	93	&	177.5	\\
Golden	7	&	361	&	73	&	4	&	\bf{	10014	}	&	 	10038	 	&	 	10051.9	 	&	765.2	&	 	10034	 	&	 	10057	 	&	89.9	&	\bf{	10014	}	&	 	10014.9	 	&	109.1	&	95	&	204.1	\\
\hline																																														
Golden	8	&	441	&	30	&	4	&	\bf{	10866	}	&	\bf{	10866	}	&	 	10885.5	 	&	1650.7	&	 	10883	 	&	 	10897.3	 	&	102.7	&	\bf{	10866	}	&	\bf{	10866	}	&	67.1	&	1004	&	1071.1	\\
Golden	8	&	441	&	32	&	4	&	\bf{	10831	}	&	\bf{	10831	}	&	 	10845.8	 	&	1700.1	&	 	10845	 	&	 	10866.9	 	&	104.2	&	\bf{	10831	}	&	\bf{	10831	}	&	67.3	&	1003	&	1070.3	\\
Golden	8	&	441	&	34	&	4	&	\bf{	10847	}	&	 	10849	 	&	 	10868.8	 	&	1565.7	&	 	10852	 	&	 	10892	 	&	104	&	\bf{	10847	}	&	\bf{	10847	}	&	65.9	&	970	&	1035.9	\\
Golden	8	&	441	&	37	&	4	&	\bf{	10859	}	&	 	10871	 	&	 	10885.1	 	&	1530.1	&	 	10883	 	&	 	10893.7	 	&	100.2	&	\bf{	10859	}	&	\bf{	10859	}	&	76.2	&	829	&	905.2	\\
Golden	8	&	441	&	41	&	4	&	\bf{	10934	}	&	\bf{	10934	}	&	 	10964.9	 	&	1660.8	&	 	10945	 	&	 	10959	 	&	94.5	&	\bf{	10934	}	&	\bf{	10934	}	&	81.8	&	792	&	873.8	\\
Golden	8	&	441	&	45	&	4	&	\bf{	10960	}	&	\bf{	10960	}	&	 	10996.5	 	&	1503.3	&	 	10969	 	&	 	10987.8	 	&	88.9	&	\bf{	10960	}	&	\bf{	10960	}	&	82.4	&	4836	&	4918.4	\\
Golden	8	&	441	&	49	&	4	&	\bf{	11042	}	&	 	11064	 	&	 	11089.5	 	&	1568.9	&	 	11060	 	&	 	11083.4	 	&	83.6	&	\bf{	11042	}	&	 	11043.3	 	&	95.8	&	127	&	222.8	\\
Golden	8	&	441	&	56	&	4	&	\bf{	11194	}	&	 	11211	 	&	 	11244.6	 	&	1678.1	&	 	11210	 	&	 	11251.9	 	&	87.7	&	\bf{	11194	}	&	 	11196.7	 	&	92.6	&	166	&	258.6	\\
Golden	8	&	441	&	63	&	4	&	\bf{	11252	}	&	 	11267	 	&	 	11312.3	 	&	1563.2	&	 	11267	 	&	 	11309	 	&	100.9	&	\bf{	11252	}	&	 	11263.6	 	&	127.4	&	155	&	282.4	\\
Golden	8	&	441	&	74	&	4	&	\bf{	11321	}	&	 	11347	 	&	 	11372	 	&	1394.3	&	 	11357	 	&	 	11371	 	&	114.4	&	\bf{	11321	}	&	 	11326.8	 	&	162.7	&	132	&	294.7	\\
Golden	8	&	441	&	89	&	4	&	\bf{	11209	}	&	 	11215	 	&	 	11258	 	&	1494.7	&	 	11243	 	&	 	11271.8	 	&	137.7	&	 	11211	 	&	 	11212.2	 	&	148.8	&	112	&	260.8	\\
\hline																																														
\end{tabular}																																														

 \end{table}
\end{landscape}

\begin{landscape} 
 \begin{table}[!htb]\center \caption{Detailed results, Golden 9--12.}\label{TabAll3} \tiny
\setlength{\extrarowheight}{1.5pt}
 \begin{tabular}{l|ccc|c|ccc|ccc|ccccc}																																														
	&	&	&	&	Exact	&	\multicolumn{3}{c|}{\ILS}	&	\multicolumn{3}{c|}{\ILSClu}	&	\multicolumn{5}{c}{\Gen}\\																																			
Inst.	&	$n$	&	$|\mathcal{C}|$	&	$m$	&	&	Best	&	Avg.	&		Avg.		&		Best		&	Avg.	&		Avg.		&		Best		&	Avg.	&		Avg.		&		Preproc.		&	Total\\					
	&	&	&	&	&	&	&Time(s)	&	&	&	Time(s)		&		&		&	Time(s)&		(s)		&		Time(s)\\																		
\hline																																														
Golden	9	&	256	&	18	&	4	&	\bf{	300	}	&	\bf{	300	}	&	\bf{	300	}	&	142.3	&	\bf{	300	}	&	\bf{	300	}	&	22.7	&	\bf{	300	}	&	\bf{	300	}	&	15	&	164	&	179	\\
Golden	9	&	256	&	19	&	4	&	\bf{	299	}	&	\bf{	299	}	&	 	299.3	 	&	135.3	&	\bf{	299	}	&	 	299.4	 	&	22	&	\bf{	299	}	&	\bf{	299	}	&	16.1	&	157	&	173.1	\\
Golden	9	&	256	&	20	&	4	&	\bf{	296	}	&	\bf{	296	}	&	\bf{	296	}	&	144.7	&	\bf{	296	}	&	 	296.3	 	&	21.9	&	\bf{	296	}	&	\bf{	296	}	&	17.6	&	133	&	150.6	\\
Golden	9	&	256	&	22	&	4	&	\bf{	290	}	&	\bf{	290	}	&	 	291	 	&	141.9	&	\bf{	290	}	&	 	291.1	 	&	21.3	&	\bf{	290	}	&	\bf{	290	}	&	18.2	&	138	&	156.2	\\
Golden	9	&	256	&	24	&	4	&	\bf{	290	}	&	\bf{	290	}	&	 	290.9	 	&	146.8	&	\bf{	290	}	&	 	291.6	 	&	20.6	&	\bf{	290	}	&	\bf{	290	}	&	19.3	&	113	&	132.3	\\
Golden	9	&	256	&	26	&	4	&	\bf{	288	}	&	\bf{	288	}	&	 	288.5	 	&	149.7	&	\bf{	288	}	&	 	290.2	 	&	18.9	&	\bf{	288	}	&	\bf{	288	}	&	21.5	&	92	&	113.5	\\
Golden	9	&	256	&	29	&	4	&	\bf{	292	}	&	\bf{	292	}	&	 	293.7	 	&	146.4	&	\bf{	292	}	&	 	294.8	 	&	19.2	&	\bf{	292	}	&	\bf{	292	}	&	23.6	&	211	&	234.6	\\
Golden	9	&	256	&	32	&	4	&	\bf{	297	}	&	\bf{	297	}	&	 	298	 	&	148.9	&	\bf{	297	}	&	 	298.3	 	&	20.9	&	\bf{	297	}	&	\bf{	297	}	&	22.5	&	79	&	101.5	\\
Golden	9	&	256	&	37	&	4	&	\bf{	294	}	&	\bf{	294	}	&	 	294.7	 	&	149.5	&	\bf{	294	}	&	 	294.5	 	&	21.1	&	\bf{	294	}	&	\bf{	294	}	&	25.9	&	68	&	93.9	\\
Golden	9	&	256	&	43	&	4	&	\bf{	295	}	&	\bf{	295	}	&	 	296.2	 	&	156.8	&	 	296	 	&	 	296.9	 	&	22.8	&	\bf{	295	}	&	 	295.8	 	&	31.6	&	59	&	90.6	\\
Golden	9	&	256	&	52	&	4	&	\bf{	296	}	&	 	297	 	&	 	298.6	 	&	166	&	 	297	 	&	 	298.2	 	&	26	&	 	297	 	&	 	297	 	&	31.6	&	33	&	64.6	\\
\hline																																														
Golden	10	&	324	&	22	&	4	&	\bf{	367	}	&	\bf{	367	}	&	 	368.6	 	&	331.3	&	 	369	 	&	 	369.9	 	&	31.2	&	\bf{	367	}	&	\bf{	367	}	&	30.3	&	230	&	260.3	\\
Golden	10	&	324	&	24	&	4	&	\bf{	361	}	&	\bf{	361	}	&	 	362.5	 	&	319.4	&	\bf{	361	}	&	 	362.2	 	&	28.7	&	\bf{	361	}	&	\bf{	361	}	&	31.6	&	175	&	206.6	\\
Golden	10	&	324	&	25	&	4	&	\bf{	359	}	&	\bf{	359	}	&	 	360.4	 	&	322.4	&	\bf{	359	}	&	 	360	 	&	29.3	&	\bf{	359	}	&	\bf{	359	}	&	31.9	&	173	&	204.9	\\
Golden	10	&	324	&	27	&	4	&	\bf{	361	}	&	\bf{	361	}	&	 	362.1	 	&	329.9	&	\bf{	361	}	&	 	362.6	 	&	29.9	&	\bf{	361	}	&	\bf{	361	}	&	33.4	&	168	&	201.4	\\
Golden	10	&	324	&	30	&	4	&	\bf{	367	}	&	\bf{	367	}	&	 	368.8	 	&	350	&	 	368	 	&	 	369	 	&	30.6	&	\bf{	367	}	&	 	367.1	 	&	40.6	&	145	&	185.6	\\
Golden	10	&	324	&	33	&	4	&	\bf{	373	}	&	 	375	 	&	 	376.6	 	&	352.8	&	 	376	 	&	 	377.2	 	&	31.1	&	\bf{	373	}	&	 	374.2	 	&	37.7	&	125	&	162.7	\\
Golden	10	&	324	&	36	&	4	&	\bf{	385	}	&	 	387	 	&	 	387.7	 	&	338.6	&	 	387	 	&	 	388.9	 	&	29.3	&	\bf{	385	}	&	 	385.2	 	&	42.1	&	144	&	186.1	\\
Golden	10	&	324	&	41	&	4	&	\bf{	400	}	&	 	401	 	&	 	402.4	 	&	330	&	\bf{	400	}	&	 	401.9	 	&	29.7	&	\bf{	400	}	&	 	400.2	 	&	51.8	&	83	&	134.8	\\
Golden	10	&	324	&	47	&	4	&	\bf{	398	}	&	 	399	 	&	 	400.1	 	&	339.3	&	 	399	 	&	 	400.1	 	&	31.6	&	\bf{	398	}	&	 	398.1	 	&	52.7	&	74	&	126.7	\\
Golden	10	&	324	&	54	&	4	&	\bf{	393	}	&	 	395	 	&	 	396.9	 	&	336.3	&	 	394	 	&	 	395.3	 	&	34.9	&	\bf{	393	}	&	 	393.6	 	&	54.5	&	72	&	126.5	\\
Golden	10	&	324	&	65	&	4	&	\bf{	387	}	&	\bf{	387	}	&	 	391.7	 	&	355.7	&	 	390	 	&	 	392	 	&	40.9	&	\bf{	387	}	&	 	387.9	 	&	75.5	&	62	&	137.5	\\
\hline																																														
Golden	11	&	400	&	27	&	5	&	\bf{	457	}	&	\bf{	457	}	&	 	458.1	 	&	596.1	&	 	458	 	&	 	459	 	&	42.4	&	\bf{	457	}	&	\bf{	457	}	&	38.3	&	238	&	313.5	\\
Golden	11	&	400	&	29	&	5	&	\bf{	455	}	&	\bf{	455	}	&	 	457.8	 	&	586.6	&	 	456	 	&	 	458.8	 	&	42.7	&	\bf{	455	}	&	\bf{	455	}	&	44.5	&	222	&	222	\\
Golden	11	&	400	&	31	&	5	&	\bf{	455	}	&	\bf{	455	}	&	 	457.1	 	&	615.1	&	 	457	 	&	 	459.1	 	&	42.2	&	\bf{	455	}	&	 	455.4	 	&	49.8	&	217	&	217	\\
Golden	11	&	400	&	34	&	5	&	\bf{	455	}	&	 	456	 	&	 	457.4	 	&	648.9	&	 	457	 	&	 	458.7	 	&	40.8	&	\bf{	455	}	&	 	455.2	 	&	49.6	&	188	&	188	\\
Golden	11	&	400	&	37	&	5	&	\bf{	459	}	&	 	460	 	&	 	461.8	 	&	693.7	&	 	461	 	&	 	462.6	 	&	41	&	\bf{	459	}	&	\bf{	459	}	&	49.9	&	163	&	163	\\
Golden	11	&	400	&	40	&	5	&	\bf{	461	}	&	 	462	 	&	 	463.8	 	&	677.7	&	 	463	 	&	 	463.5	 	&	41.3	&	\bf{	461	}	&	\bf{	461	}	&	48.4	&	142	&	180.3	\\
Golden	11	&	400	&	45	&	5	&	\bf{	462	}	&	 	464	 	&	 	465.3	 	&	636.4	&	 	465	 	&	 	465.7	 	&	43.5	&	\bf{	462	}	&	 	462.1	 	&	52.7	&	270	&	314.5	\\
Golden	11	&	400	&	50	&	5	&	\bf{	458	}	&	 	459	 	&	 	461.4	 	&	677.3	&	 	460	 	&	 	461.2	 	&	45.9	&	\bf{	458	}	&	 	458.5	 	&	57.8	&	123	&	172.8	\\
Golden	11	&	400	&	58	&	5	&	\bf{	456	}	&	 	458	 	&	 	459.9	 	&	701.1	&	 	460	 	&	 	462	 	&	49.6	&	\bf{	456	}	&	 	456.4	 	&	80.3	&	105	&	154.6	\\
Golden	11	&	400	&	67	&	5	&	\bf{	454	}	&	 	457	 	&	 	459.2	 	&	733.8	&	 	458	 	&	 	460.5	 	&	59.7	&	\bf{	454	}	&	 	454.6	 	&	85.7	&	91	&	140.9	\\
Golden	11	&	400	&	80	&	5	&	\bf{	451	}	&	 	455	 	&	 	457.2	 	&	676.3		& 	456	 	&	 	458.6	 	&	67	&	\bf{	451	}	&	 	451.7	 	&	98.8	&	63	&	111.4	\\
\hline																																														
Golden	12	&	484	&	33	&	5	&	\bf{	535	}	&	\bf{	535	}	&	 	538.1	 	&	1338.3	&	 	537	 	&	 	538.4	 	&	63	&	\bf{	535	}	&	\bf{	535	}	&	80.3	&	270	&	322.7	\\
Golden	12	&	484	&	35	&	5	&	\bf{	537	}	&	\bf{	537	}	&	 	538.4	 	&	1367.3	&	\bf{	537	}	&	 	539.7	 	&	60.4	&	\bf{	537	}	&	\bf{	537	}	&	62.6	&	255	&	312.8	\\
Golden	12	&	484	&	38	&	5	&	\bf{	535	}	&	\bf{	535	}	&	 	540	 	&	1327.9	&	 	537	 	&	 	540.5	 	&	61.9	&	\bf{	535	}	&	\bf{	535	}	&	70.3	&	231	&	311.3	\\
Golden	12	&	484	&	41	&	5	&	\bf{	537	}	&	 	539	 	&	 	542.5	 	&	1516.8	&	 	538	 	&	 	542.7	 	&	63.9	&	\bf{	537	}	&	\bf{	537	}	&	70.9	&	200	&	285.7	\\
Golden	12	&	484	&	44	&	5	&	\bf{	535	}	&	 	537	 	&	 	540.3	 	&	1406.3	&	 	540	 	&	 	542.2	 	&	62.4	&	\bf{	535	}	&	 	535.8	 	&	87.4	&	194	&	292.8	\\
Golden	12	&	484	&	49	&	5	&	\bf{	533	}	&	 	537	 	&	 	538.3	 	&	1443	&	 	538	 	&	 	540.9	 	&	68.9	&	\bf{	533	}	&	 	533.4	 	&	104	&	192	&	272.3	\\
Golden	12	&	484	&	54	&	5	&	\bf{	535	}	&	 	537		&	 	540.2	 	&	1679.6	&	 	539	 	&	 	542.9	 	&	69.5	&	\bf{	535	}	&	 	535.3	 	&	92.1	&	1656	&	1718.6	\\
Golden	12	&	484	&	61	&	5	&		538		&		538		&		541.3	 	&	1453.1	&		539	 	&	 	544.8	 	&	72.7	&	\bf{\underline{	535	}}	&		535.2		&	89.4	&	152	&	222.3	\\
Golden	12	&	484	&	70	&	5	&		546		&		538		&		540		&	1519.4	&		540		&		543.6		&	80.5	&	\bf{\underline{	533	}}	&		533.5		&	104.4	&	131	&	201.9	\\
Golden	12	&	484	&	81	&	5	&		546		&		541		&		543.4		&	1296.2	&		544		&		547.5		&	85.4	&	\bf{\underline{	535	}}	&		536		&	133.7	&	93	&	180.4	\\
Golden	12	&	484	&	97	&	5	&		560		&		548		&		553		&	1275.5	&		548		&		554.5		&	105.9	&	\bf{\underline{	544	}}	&		544.7		&	139.6	&	56	&	160	\\
\hline																																														
\end{tabular}																																														

 \end{table} 
\end{landscape}
\begin{landscape} 
 \begin{table}[!htb]\center \caption{Detailed results, Golden 13--16.}\label{TabAll4} \tiny
\setlength{\extrarowheight}{1.5pt}
 \begin{tabular}{l|ccc|c|ccc|ccc|ccccc}																																														
	&	&	&	&	Exact	&	\multicolumn{3}{c|}{\ILS}	&	\multicolumn{3}{c|}{\ILSClu}	&	\multicolumn{5}{c}{\Gen}\\
	Inst.	&	$n$	&	$|\mathcal{C}|$	&	$m$	&	&	Best	&	Avg.	&		Avg.		&		Best		&	Avg.	&		Avg.		&		Best		&	Avg.	&		Avg.		&		Preproc.		&	Total\\		&	&	&	&	&	&	&Time(s)	&	&	&	Time(s)		&		&		&	Time(s)&		(s)		&		Time(s)\\																		
\hline																																														
Golden	13	&	253	&	17	&	4	&	\textbf{	552	}	&	\bf{	552	}	&	\bf{	552	}	&	144.7	&	\bf{	552	}	&	 	553.8	 	&	21	&	\bf{	552	}	&	\bf{	552	}	&	15.8	&	159	&	251.1	\\
Golden	13	&	253	&	19	&	4	&	\textbf{	549	}	&	\bf{	549	}	&	 	549.1	 	&	138.4	&	 	551	 	&	 	551.3	 	&	20.7	&	\bf{	549	}	&	\bf{	549	}	&	16.4	&	128	&	217.4	\\
Golden	13	&	253	&	20	&	4	&	\textbf{	548	}	&	\bf{	548	}	&	 	548.3	 	&	140.9	&	\bf{	548	}	&	 	549.5	 	&	19.7	&	\bf{	548	}	&	\bf{	548	}	&	17.9	&	127	&	231.4	\\
Golden	13	&	253	&	22	&	4	&	\textbf{	548	}	&	\bf{	548	}	&	 	548.4	 	&	141.2	&	 	549	 	&	 	549.5	 	&	20.5	&	\bf{	548	}	&	\bf{	548	}	&	20.3	&	110	&	243.7	\\
Golden	13	&	253	&	23	&	4	&	\textbf{	548	}	&	\bf{	548	}	&	 	548.5	 	&	138.5	&	\bf{	548	}	&	 	549.1	 	&	20.1	&	\bf{	548	}	&	\bf{	548	}	&	19.6	&	107	&	246.6	\\
Golden	13	&	253	&	26	&	4	&	\textbf{	542	}	&	\bf{	542	}	&	 	542.1	 	&	146.6	&	\bf{	542	}	&	 	542.6	 	&	20.2	&	\bf{	542	}	&	\bf{	542	}	&	20.4	&	94	&	109.8	\\
Golden	13	&	253	&	29	&	4	&	\textbf{	540	}	&	\bf{	540	}	&	 	540.3	 	&	147.7	&	\bf{	540	}	&	 	540.9	 	&	21.3	&	\bf{	540	}	&	\bf{	540	}	&	21.8	&	218	&	234.4	\\
Golden	13	&	253	&	32	&	4	&	\textbf{	543	}	&	\bf{	543	}	&	 	543.7	 	&	145.7	&	 	544	 	&	 	544.9	 	&	21.3	&	\bf{	543	}	&	\bf{	543	}	&	23.7	&	72	&	89.9	\\
Golden	13	&	253	&	37	&	4	&	\textbf{	545	}	&	 	546	 	&	 	547.9	 	&	147.4	&	 	546	 	&	 	548.8	 	&	21.7	&	\bf{	545	}	&	 	545.2	 	&	33.7	&	54	&	74.3	\\
Golden	13	&	253	&	43	&	4	&	\textbf{	553	}	&	 	554	 	&	 	555.1	 	&	146.4	&	 	554	 	&	 	555.6	 	&	25.3	&	\bf{	553	}	&	\bf{	553	}	&	29.2	&	44	&	63.6	\\
Golden	13	&	253	&	51	&	4	&	\textbf{	560	}	&	 	561	 	&	 	562.1	 	&	164.9	&	 	561	 	&	 	563	 	&	28.6	&	\bf{	560	}	&	 	560.4	 	&	37.9	&	29	&	49.4	\\
\hline																																														
Golden	14	&	321	&	22	&	4	&	\textbf{	692	}	&	\bf{	692	}	&	 	692.8	 	&	320.7	&	 	693	 	&	 	695	 	&	35.4	&	\bf{	692	}	&	\bf{	692	}	&	25.6	&	214	&	235.8	\\
Golden	14	&	321	&	23	&	4	&	\textbf{	688	}	&	\bf{	688	}	&	 	688.3	 	&	330.6	&	\bf{	688	}	&	 	689.7	 	&	32.2	&	\bf{	688	}	&	\bf{	688	}	&	27.3	&	181	&	204.7	\\
Golden	14	&	321	&	25	&	4	&	\textbf{	678	}	&	\bf{	678	}	&	 	679.1	 	&	317.3	&	 	679	 	&	 	680	 	&	31.2	&	\bf{	678	}	&	\bf{	678	}	&	29	&	169	&	202.7	\\
Golden	14	&	321	&	27	&	4	&	\textbf{	676	}	&	\bf{	676	}	&	 	677.6	 	&	317.3	&	\bf{	676	}	&	 	678.6	 	&	31.1	&	\bf{	676	}	&	\bf{	676	}	&	29.3	&	147	&	176.2	\\
Golden	14	&	321	&	30	&	4	&	\textbf{	678	}	&	 	680	 	&	 	680.5	 	&	319.3	&	 	682	 	&	 	682.3	 	&	30.5	&	\bf{	678	}	&	\bf{	678	}	&	32.2	&	128	&	165.9	\\
Golden	14	&	321	&	33	&	4	&	\textbf{	682	}	&	\bf{	682	}	&	 	683.6	 	&	341.5	&	 	684	 	&	 	685.3	 	&	31	&	\bf{	682	}	&	\bf{	682	}	&	32.9	&	118	&	143.6	\\
Golden	14	&	321	&	36	&	4	&	\textbf{	687	}	&	\bf{	687	}	&	 	688.3	 	&	349.3	&	 	688	 	&	 	689.5	 	&	33	&	\bf{	687	}	&	\bf{	687	}	&	33.7	&	163	&	190.3	\\
Golden	14	&	321	&	41	&	4	&	\textbf{	690	}	&	 	691	 	&	 	692.4	 	&	339.7	&	 	691	 	&	 	692.9	 	&	33.7	&	\bf{	690	}	&	 	690.1	 	&	50.2	&	83	&	112	\\
Golden	14	&	321	&	46	&	4	&	\textbf{	694	}	&	 	697	 	&	 	698	 	&	352	&	 	697	 	&	 	699	 	&	34.1	&	\bf{	694	}	&	 	695.7	 	&	48.5	&	65	&	94.3	\\
Golden	14	&	321	&	54	&	4	&	\textbf{	699	}	&	 	701	 	&	 	703	 	&	339	&	 	703	 	&	 	704.2	 	&	38.2	&	\bf{	699	}	&	 	700.1	 	&	52.7	&	53	&	85.2	\\
Golden	14	&	321	&	65	&	4	&	\textbf{	703	}	&	\bf{	703	}	&	 	706.3	 	&	342.7	&	 	704	 	&	 	705.9	 	&	47	&	\bf{	703	}	&	\bf{	703	}	&	58	&	37	&	69.9	\\
\hline																																														
Golden	15	&	397	&	27	&	4	&	\textbf{	842	}	&	 	844	 	&	 	844.2	 	&	728.5	&	 	844	 	&	 	846.3	 	&	54.9	&	\bf{	842	}	&	\bf{	842	}	&	47.5	&	244	&	277.7	\\
Golden	15	&	397	&	29	&	4	&	\textbf{	843	}	&	 	844	 	&	 	846.4	 	&	736.6	&	 	845	 	&	 	849.2	 	&	52.2	&	\bf{	843	}	&	 	843.4	 	&	52.5	&	218	&	268.2	\\
Golden	15	&	397	&	31	&	4	&	\textbf{	837	}	&	 	839	 	&	 	840.9	 	&	742.7	&	 	841	 	&	 	843.3	 	&	48.5	&	\bf{	837	}	&	 	837.1	 	&	49.3	&	203	&	251.5	\\
Golden	15	&	397	&	34	&	4	&	\textbf{	838	}	&	 	842	 	&	 	844.2	 	&	755.2	&	 	844	 	&	 	846.6	 	&	48.2	&	\bf{	838	}	&	 	838.5	 	&	72.4	&	169	&	221.7	\\
Golden	15	&	397	&	37	&	4	&	\textbf{	845	}	&	 	848	 	&	 	850.1	 	&	739.6	&	 	848	 	&	 	853.5	 	&	48.5	&	\bf{	845	}	&	 	845.2	 	&	51.6	&	152	&	210	\\
Golden	15	&	397	&	40	&	4	&	\textbf{	849	}	&	\bf{	849	}	&	 	851.9	 	&	744.5	&	 	850	 	&	 	853	 	&	50	&	\bf{	849	}	&	 	849.2	 	&	65.3	&	141	&	188.5	\\
Golden	15	&	397	&	45	&	5	&	\textbf{	853	}	&	 	856	 	&	 	857.8	 	&	655.4	&	 	855	 	&	 	858.4	 	&	45.5	&	\bf{	853	}	&	 	853.1	 	&	58.5	&	1033	&	1085.5	\\
Golden	15	&	397	&	50	&	5	&	\textbf{	851	}	&	 	854	 	&	 	855.7	 	&	710.7	&	 	857	 	&	 	859	 	&	48.8	&	\bf{	851	}	&	 	851.8	 	&	68.1	&	111	&	160.3	\\
Golden	15	&	397	&	57	&	5	&	\textbf{	850	}	&	 	854	 	&	 	856.4	 	&	700.3	&	 	856	 	&	 	858.3	 	&	50.6	&	\bf{	850	}	&	 	850.4	 	&	83.5	&	94	&	166.4	\\
Golden	15	&	397	&	67	&	5	&	\textbf{	855	}	&	 	857	 	&	 	862.1	 	&	685.4	&	 	861	 	&	 	862.9	 	&	57.2	&	 	857	 	&	 	857.3	 	&	82.9	&	73	&	124.6	\\
Golden	15	&	397	&	80	&	5	&	\textbf{	857	}	&	 	863	 	&	 	864.3	 	&	648.1	&	 	862	 	&	 	864.2	 	&	70.9	&	 	858	 	&	 	859.6	 	&	94.3	&	51	&	116.3	\\
\hline																																														
Golden	16	&	481	&	33	&	5	&	\textbf{	1030	}	&	\bf{	1030	}	&	\bf{	1030	}	&	1268.8	&	\bf{	1030	}	&	 	1031.2	 	&	67.7	&	\bf{	1030	}	&	\bf{	1030	}	&	54.7	&	296	&	354.5	\\
Golden	16	&	481	&	35	&	5	&	\textbf{	1028	}	&	\bf{	1028	}	&	 	1029.4	 	&	1238	&	\bf{	1028	}	&	 	1030.6	 	&	65.5	&	\bf{	1028	}	&	\bf{	1028	}	&	59.5	&	251	&	319.1	\\
Golden	16	&	481	&	37	&	5	&	\textbf{	1028	}	&	\bf{	1028	}	&	 	1029	 	&	1288.6	&	\bf{	1028	}	&	 	1029.7	 	&	63.1	&	\bf{	1028	}	&	\bf{	1028	}	&	60.9	&	238	&	321.5	\\
Golden	16	&	481	&	41	&	5	&	\textbf{	1032	}	&	 	1033	 	&	 	1034.3	 	&	1329.9	&	 	1033	 	&	 	1034.6	 	&	64.8	&	\bf{	1032	}	&	\bf{	1032	}	&	60.1	&	199	&	281.9	\\
Golden	16	&	481	&	44	&	5	&	\textbf{	1028	}	&	 	1029	 	&	 	1031.2	 	&	1297	&	 	1032	 	&	 	1033.3	 	&	64.6	&	\bf{	1028	}	&	\bf{	1028	}	&	63.8	&	179	&	273.3	\\
Golden	16	&	481	&	49	&	5	&	\textbf{	1031	}	&	 	1033	 	&	 	1034.6	 	&	1264.7	&	 	1034	 	&	 	1036	 	&	65.4	&	\bf{	1031	}	&	\bf{	1031	}	&	71.2	&	161	&	215.7	\\
Golden	16	&	481	&	54	&	5	&	\textbf{	1022	}	&	 	1024		&	 	1026		&	1505.4	&	 	1027		&	 	1028.8		&	69.2	&	\bf{	1022	}	&	\bf{	1022	}	&	83.9	&	214	&	273.5	\\
Golden	16	&	481	&	61	&	5	&	\textbf{	1013	}	&		1015		&		1018.7		&	1498.7	&		1022		&		1023.1		&	74	&	\bf{	1013	}	&	 	1013.8	 	&	94.7	&	143	&	203.9	\\
Golden	16	&	481	&	69	&	5	&	\textbf{	1012	}	&		1015		&		1017.7		&	1525.2	&		1020		&		1020.8		&	79.4	&	\bf{	1012	}	&	 	1012.3	 	&	114.5	&	120	&	180.1	\\
Golden	16	&	481	&	81	&	5	&	\textbf{	1018	}	&		1019		&		1023.7		&	1312	&		1023		&		1026.5		&	93.2	&	\bf{	1018	}	&	\bf{	1018	}	&	105.3	&	94	&	157.8	\\
Golden	16	&	481	&	97	&	5	&	\textbf{	1018	}	&		1023		&		1027.1		&	1258.1	&		1027		&		1029.9		&	110.9	&	\bf{	1018	}	&	 	1020	 	&	158.8	&	57	&	128.2	\\
\hline																	
\end{tabular}

 \end{table} 
\end{landscape} 
 \begin{landscape} 
 \begin{table}[!htb]\center \caption{Detailed results, Golden 17--20.}\label{TabAll5} \tiny
\setlength{\extrarowheight}{1.5pt}
 \begin	{tabular}{l|ccc|c|ccc|ccc|ccccc}																		&		&		&		&		Exact	&	\multicolumn		{3}{c|}{\ILS}								&		\multicolumn{3}{c|}{\ILSClu}									&	\multicolumn{5}{c}{\Gen}\\		
Inst.	& $n$	&	$|\mathcal{C}|$	&	$m$	&			&	Best	&		Avg.		&		Avg.		&	Best	&		Avg.		&		Avg. &		Best	&	Avg.		&		Avg. &		Preproc.			&	Total\\		
	&   &			&	&		&		&           &Time(s)   &		&		&		Time(s)		&		&	& Time(s)&   (s)				& Time(s)\\	
	\hline																	
Golden	17	&	241	&	17	&	3	&	\textbf{	418	}	&	\bf{	418	}	&	\bf{	418	}	&	177.8	&	\bf{	418	}	&	 	418.1	 	&	32.3	&	\bf{	418	}	&	\bf{	418	}	&	15.4	&	209	&	292.9	\\
Golden	17	&	241	&	18	&	3	&	\textbf{	419	}	&	\bf{	419	}	&	\bf{	419	}	&	176.2	&	\bf{	419	}	&	\bf{	419	}	&	30.9	&	\bf{	419	}	&	\bf{	419	}	&	17.2	&	192	&	286.7	\\
Golden	17	&	241	&	19	&	3	&	\textbf{	422	}	&	\bf{	422	}	&	\bf{	422	}	&	169.4	&	\bf{	422	}	&	\bf{	422	}	&	31	&	\bf{	422	}	&	\bf{	422	}	&	17.8	&	172	&	286.5	\\
Golden	17	&	241	&	21	&	3	&	\textbf{	425	}	&	\bf{	425	}	&	\bf{	425	}	&	171.1	&	\bf{	425	}	&	\bf{	425	}	&	30.2	&	\bf{	425	}	&	\bf{	425	}	&	20	&	162	&	267.3	\\
Golden	17	&	241	&	22	&	3	&	\textbf{	424	}	&	\bf{	424	}	&	\bf{	424	}	&	179.5	&	\bf{	424	}	&	 	424.1	 	&	31.3	&	\bf{	424	}	&	\bf{	424	}	&	20.1	&	155	&	313.8	\\
Golden	17	&	241	&	25	&	3	&	\textbf{	418	}	&	\bf{	418	}	&	\bf{	418	}	&	173.3	&	\bf{	418	}	&	 	418.4	 	&	25.9	&	\bf{	418	}	&	\bf{	418	}	&	21.9	&	111	&	126.4	\\
Golden	17	&	241	&	27	&	3	&	\textbf{	414	}	&	\bf{	414	}	&	\bf{	414	}	&	165.3	&	\bf{	414	}	&	\bf{	414	}	&	25	&	\bf{	414	}	&	\bf{	414	}	&	22.9	&	81	&	98.2	\\
Golden	17	&	241	&	31	&	4	&	\textbf{	421	}	&	\bf{	421	}	&	 	421.1	 	&	132	&	\bf{	421	}	&	 	421.3	 	&	20.9	&	\bf{	421	}	&	\bf{	421	}	&	21	&	73	&	90.8	\\
Golden	17	&	241	&	35	&	4	&	\textbf{	417	}	&	\bf{	417	}	&	 	417.1	 	&	135.9	&	\bf{	417	}	&	 	417.4	 	&	20.8	&	\bf{	417	}	&	\bf{	417	}	&	22	&	53	&	73	\\
Golden	17	&	241	&	41	&	4	&	\textbf{	412	}	&	\bf{	412	}	&	 	412.1	 	&	134.7	&	\bf{	412	}	&	\bf{	412	}	&	23.8	&	\bf{	412	}	&	\bf{	412	}	&	24.6	&	36	&	56.1	\\
Golden	17	&	241	&	49	&	4	&	\textbf{	414	}	&	\bf{	414	}	&	 	414.1	 	&	138.5	&	\bf{	414	}	&	 	414.7	 	&	27	&	\bf{	414	}	&	\bf{	414	}	&	27.3	&	32	&	53.9	\\
\hline																																														
Golden	18	&	301	&	21	&	4	&	\textbf{	592	}	&	\bf{	592	}	&	\bf{	592	}	&	304.1	&	\bf{	592	}	&	 	593.6	 	&	41.1	&	\bf{	592	}	&	\bf{	592	}	&	22	&	329	&	351.9	\\
Golden	18	&	301	&	22	&	4	&	\textbf{	594	}	&	\bf{	594	}	&	\bf{	594	}	&	318.3	&	\bf{	594	}	&	 	595.6	 	&	39	&	\bf{	594	}	&	\bf{	594	}	&	22.5	&	300	&	321	\\
Golden	18	&	301	&	24	&	4	&	\textbf{	592	}	&	\bf{	592	}	&	 	592.1	 	&	323.4	&	\bf{	592	}	&	 	593.7	 	&	41.2	&	\bf{	592	}	&	\bf{	592	}	&	23	&	294	&	316	\\
Golden	18	&	301	&	26	&	4	&	\textbf{	590	}	&	\bf{	590	}	&	\bf{	590	}	&	319.6	&	\bf{	590	}	&	 	590.9	 	&	36	&	\bf{	590	}	&	\bf{	590	}	&	24.5	&	229	&	253.6	\\
Golden	18	&	301	&	28	&	4	&	\textbf{	577	}	&	\bf{	577	}	&	\bf{	577	}	&	317.9	&	\bf{	577	}	&	 	577.4	 	&	35.6	&	\bf{	577	}	&	\bf{	577	}	&	26.3	&	164	&	191.3	\\
Golden	18	&	301	&	31	&	4	&	\textbf{	578	}	&	\bf{	578	}	&	\bf{	578	}	&	302.2	&	\bf{	578	}	&	 	578.7	 	&	35.3	&	\bf{	578	}	&	\bf{	578	}	&	28.8	&	136	&	158	\\
Golden	18	&	301	&	34	&	4	&	\textbf{	582	}	&	\bf{	582	}	&	 	582.1	 	&	305.7	&	\bf{	582	}	&	 	582.1	 	&	34.2	&	\bf{	582	}	&	\bf{	582	}	&	29.6	&	112	&	134.5	\\
Golden	18	&	301	&	38	&	4	&	\textbf{	586	}	&	 	587	 	&	 	587.3	 	&	301.8	&	\bf{	586	}	&	 	587	 	&	35.6	&	\bf{	586	}	&	\bf{	586	}	&	34.1	&	100	&	123	\\
Golden	18	&	301	&	43	&	4	&	\textbf{	594	}	&	\bf{	594	}	&	 	594.8	 	&	303.5	&	\bf{	594	}	&	 	594.5	 	&	36	&	\bf{	594	}	&	\bf{	594	}	&	35.5	&	77	&	101.5	\\
Golden	18	&	301	&	51	&	4	&	\textbf{	601	}	&	\bf{	601	}	&	 	601.9	 	&	309.5	&	\bf{	601	}	&	 	601.9	 	&	39.4	&	\bf{	601	}	&	\bf{	601	}	&	43.6	&	51	&	77.3	\\
Golden	18	&	301	&	61	&	4	&	\textbf{	599	}	&	\bf{	599	}	&	 	599.6	 	&	299.5	&	\bf{	599	}	&	 	600.1	 	&	43.8	&	\bf{	599	}	&	\bf{	599	}	&	48.4	&	47	&	75.8	\\
\hline																																														
Golden	19	&	361	&	25	&	10	&	\textbf{	925	}	&	\bf{	925	}	&	\bf{	925	}	&	342.5	&	 	926	 	&	 	926.4	 	&	55.1	&	\bf{	925	}	&	\bf{	925	}	&	13.8	&	607	&	636.6	\\
Golden	19	&	361	&	26	&	10	&	\textbf{	924	}	&	\bf{	924	}	&	\bf{	924	}	&	335.9	&	\bf{	924	}	&	 	925.2	 	&	52.2	&	\bf{	924	}	&	\bf{	924	}	&	13.9	&	519	&	553.1	\\
Golden	19	&	361	&	28	&	4	&	\textbf{	808	}	&	\bf{	808	}	&	 	808.6	 	&	658.1	&	 	809	 	&	 	810.2	 	&	63.3	&	\bf{	808	}	&	\bf{	808	}	&	35.1	&	389	&	424.5	\\
Golden	19	&	361	&	31	&	4	&	\textbf{	811	}	&	\bf{	811	}	&	 	812	 	&	597.7	&	 	812	 	&	 	813	 	&	56.7	&	\bf{	811	}	&	 	811.2	 	&	48.6	&	288	&	331.6	\\
Golden	19	&	361	&	33	&	4	&	\textbf{	797	}	&	\bf{	797	}	&	 	798.3	 	&	609.5	&	\bf{	797	}	&	 	798.3	 	&	51.4	&	\bf{	797	}	&	\bf{	797	}	&	43.3	&	201	&	249.4	\\
Golden	19	&	361	&	37	&	5	&	\textbf{	799	}	&	\bf{	799	}	&	 	799.9	 	&	539.1	&	\bf{	799	}	&	 	800.4	 	&	46.3	&	\bf{	799	}	&	\bf{	799	}	&	38.1	&	147	&	160.8	\\
Golden	19	&	361	&	41	&	5	&	\textbf{	789	}	&	\bf{	789	}	&	\bf{	789	}	&	532.6	&	\bf{	789	}	&	\bf{	789	}	&	45.6	&	\bf{	789	}	&	\bf{	789	}	&	36.1	&	123	&	136.9	\\
Golden	19	&	361	&	46	&	5	&	\textbf{	788	}	&	\bf{	788	}	&	\bf{	788	}	&	529.3	&	\bf{	788	}	&	\bf{	788	}	&	44.3	&	\bf{	788	}	&	\bf{	788	}	&	37.7	&	116	&	151.1	\\
Golden	19	&	361	&	52	&	5	&	\textbf{	800	}	&	\bf{	800	}	&	\bf{	800	}	&	536.6	&	\bf{	800	}	&	 	800.2	 	&	47.5	&	\bf{	800	}	&	\bf{	800	}	&	42.9	&	109	&	157.6	\\
Golden	19	&	361	&	61	&	5	&	\textbf{	807	}	&	\bf{	807	}	&	 	808.6	 	&	521.5	&	\bf{	807	}	&	 	807.8	 	&	51.6	&	\bf{	807	}	&	\bf{	807	}	&	48.4	&	85	&	128.3	\\
Golden	19	&	361	&	73	&	5	&	\textbf{	810	}	&	 	811	 	&	 	812.5	 	&	553.6	&	 	811	 	&	 	812.1	 	&	62.7	&	\bf{	810	}	&	 	810.1	 	&	67.2	&	71	&	109.1	\\
\hline																																														
Golden	20	&	421	&	29	&	11	&	\textbf{	1220	}	&	\bf{	1220	}	&	\bf{	1220	}	&	529	&	 	1221	 	&	 	1226.1	 	&	74	&	\bf{	1220	}	&	\bf{	1220	}	&	21	&	821	&	857.1	\\
Golden	20	&	421	&	31	&	12	&	\textbf{	1232	}	&	\bf{	1232	}	&	 	1232.1	 	&	506.2	&	 	1235	 	&	 	1237.2	 	&	73.8	&	\bf{	1232	}	&	\bf{	1232	}	&	20.7	&	536	&	573.7	\\
Golden	20	&	421	&	33	&	12	&	\textbf{	1208	}	&	\bf{	1208	}	&	 	1208.2	 	&	490.5	&	 	1212	 	&	 	1212.9	 	&	68.3	&	\bf{	1208	}	&	\bf{	1208	}	&	22.2	&	444	&	486.9	\\
Golden	20	&	421	&	36	&	5	&	\textbf{	1059	}	&	 	1060	 	&	 	1062	 	&	1045.9	&	 	1060	 	&	 	1064.8	 	&	85.5	&	\bf{	1059	}	&	\bf{	1059	}	&	44.1	&	394	&	442.4	\\
Golden	20	&	421	&	39	&	5	&	\textbf{	1052	}	&	\bf{	1052	}	&	 	1053.2	 	&	1083	&	\bf{	1052	}	&	 	1054.5	 	&	78.3	&	\bf{	1052	}	&	\bf{	1052	}	&	45	&	260	&	327.2	\\
Golden	20	&	421	&	43	&	5	&	\textbf{	1052	}	&	\bf{	1052	}	&	 	1053.8	 	&	1034.2	&	 	1053	 	&	 	1055.6	 	&	73.7	&	\bf{	1052	}	&	\bf{	1052	}	&	48.8	&	209	&	230	\\
Golden	20	&	421	&	47	&	5	&	\textbf{	1053	}	&	 	1054	 	&	 	1055.2	 	&	1038.1	&	 	1055	 	&	 	1056	 	&	73.4	&	\bf{	1053	}	&	\bf{	1053	}	&	60.1	&	78	&	98.7	\\
Golden	20	&	421	&	53	&	5	&	\textbf{	1058	}	&		1058		&		1059.8		&	834	&		1060		&		1061.1		&	75.6	&	\bf{	1058	}	&	\bf{	1058	}	&	55.9	&	207	&	229.2	\\
Golden	20	&	421	&	61	&	5	&	\textbf{	1058	}	&		1058		&		1059.7		&	1050.4	&		1060		&		1061.1		&	76.1	&	\bf{	1058	}	&	\bf{	1058	}	&	62.4	&	203	&	247.1	\\
Golden	20	&	421	&	71	&	5	&	\textbf{	1049	}	&		1059		&		1060.4		&	1049.1	&		1059		&		1061.1		&	83.5	&	 	1059	 	&	 	1059	 	&	71.5	&	186	&	231	\\
Golden	20	&	421	&	85	&	5	&	\textbf{	1049	}	&		1050		&		1052		&	1086.5	&		1050		&		1051.5		&	91.3	&	\bf{	1049	}	&	\bf{	1049	}	&	93.9	&	97	&	145.8	\\
\hline																													\end{tabular}																																													

 \end{table} 
\end{landscape}

\begin{landscape} 
 \begin{table}[!htb]\center \caption{Detailed results, GVRP2}\label{TabAllGVRP2} \tiny
\setlength{\extrarowheight}{1.5pt}
 \begin{tabular}{l|ccc|c|ccc|ccc|ccccc}	&	&	&	&	Exact	&	\multicolumn{3}{c|}{\ILS}	&	\multicolumn{3}{c|}{\ILSClu}	&	\multicolumn{5}{c}{\Gen}\\		Inst.		&		$n$		&	$|\mathcal{C}|$	&		$m$		&		&		Best	&	Avg.		&		Avg.		&		Best	&	Avg.	&	Avg.	&	Best	&	Avg.	&	Avg.	&	Preproc.	&	Total\\	&	&	&	&	&	&	&Time(s)	&	&	&	Time(s)	&	&	&	Time(s)&	(s)	&	Time(s)\\
\hline
G	&	262	&	131	&	12	&	--			&		3736		&		3749.6		&	127.1	&		3709		&		3718		&	66.8	&		\textbf{\underline{3693}}		&		3711		&	230.9	&	9	&	239.9	\\																													
CMT	&	101	&	51	&	5	&	\bf{	642	}	&	\bf{	642	}	&	\bf{	642	}	&	5.1	&	\bf{	642	}	&	\bf{	642	}	&	8.5	&	\bf{	642	}	&		642		&	9.2	&	3	&	12.2	\\																													
CMT	&	121	&	61	&	4	&	\bf{	807	}	&	\bf{	807	}	&		809.2		&	15.7	&	\bf{	807	}	&	\bf{	807	}	&	16.3	&	\bf{	807	}	&		808.3		&	22.8	&	4	&	26.8	\\																													
CMT	&	151	&	76	&	6	&	\bf{	816	}	&	\bf{	816	}	&		817.5		&	22.2	&	\bf{	816	}	&	\bf{	816	}	&	21.8	&	\bf{	816	}	&		817.6		&	27.4	&	4	&	31.4	\\																													
CMT	&	200	&	100	&	8	&		994		&		965		&		971.6		&	54.3	&		\bf{\underline{955}}		&		961.4		&	35.6	&		960		&		964.2		&	89.8	&	8	&	97.8	\\
\hline	\end{tabular}

 \end{table}  

 \begin{table}[!htb]\center \caption{Detailed results, GVRP3}\label{TabAllGVRP3} \tiny
\setlength{\extrarowheight}{1.5pt}
 \begin{tabular}{l|ccc|c|ccc|ccc|ccccc}	&	&	&	&	Exact	&	\multicolumn{3}{c|}{\ILS}	&	\multicolumn{3}{c|}{\ILSClu}	&	\multicolumn{5}{c}{\Gen}\\																																		
Inst.	&	$n$	&	$|\mathcal{C}|$	&	$m$	&	&	Best	&	Avg.		&		Avg.		&		Best	&	Avg.		&		Avg.		&		Best	&	Avg.		&		Avg.		&		Preproc.	&	Total\\				
	&	&	&	&	&	&	&Time(s)	&	&	&		Time(s)		&		&	&	Time(s)&		(s)		&		Time(s)\\																
\hline																																													
G	&	262	&	88	&	9	&	3596			&		3294		&		3305.9		&	128	&		\bf{\underline{3290}}		&		3291.8		&	40.2	&		\bf{\underline{3290}}		&		3293.2		&	69.6	&	21	&	90.6	\\
CMT	&	101	&	34	&	4	&	\bf{	607	}	&	\bf{	607	}	&	\bf{	607	}	&	7.4	&	\bf{	607	}	&	\bf{	607	}	&	7.1	&	\bf{	607	}	&	\bf{	607	}	&	7.3	&	8	&	15.3	\\
CMT	&	121	&	41	&	3	&	\bf{	691	}	&		694		&		695.4		&	14.3	&	\bf{	691	}	&	\bf{	691	}	&	12.6	&	\bf{	691	}	&	\bf{	691	}	&	17.8	&	9	&	26.8	\\
CMT	&	151	&	51	&	4	&	\bf{	804	}	&	\bf{	804	}	&	\bf{	804	}	&	29.5	&	\bf{	804	}	&	\bf{	804	}	&	16.2	&	\bf{	804	}	&	\bf{	804	}	&	18.4	&	12	&	30.4	\\
CMT	&	200	&	67	&	6	&	\bf{	908	}	&	\bf{	908	}	&	\bf{	908	}	&	59.5	&	\bf{	908	}	&	\bf{	908	}	&	24.4	&	\bf{	908	}	&	\bf{	908	}	&	24.1	&	16	&	40.1	\\
\hline	\end{tabular}

 \end{table}  
 
  \begin{table}[!htb]\center \caption{Detailed results, Li}\label{TabAllLi} \tiny
\setlength{\extrarowheight}{1.5pt}
 \begin{tabular}{l|ccc|ccc|ccc|ccccc}	&	&	&	&	\multicolumn{3}{c|}{\ILS}	&	\multicolumn{3}{c|}{\ILSClu}	&	\multicolumn{5}{c}{\Gen}\\	Inst.	&	$n$	&	$|\mathcal{C}|$	&	$m$	&	Best	&	Avg.	&	Avg.	&	Best	&	Avg.	&	Avg.	&	Best	&	Avg.	&	Avg.	&	Preproc.	&	Total\\	&	&	&	&	&	&Time(s)	&	&	&	Time(s)	&	&	&	Time(s)&	(s)	&	Time(s)\\
\hline	Li	&	560	&	113	&	39	&	\bf{\underline{27962}}	&	28079.7	&	714.6	&	27970	&	27997.1	&	122.7	&	\bf{\underline{27962}}	&	27963	&	83.9	&	200	&	283.9	\\																												
Li	&	600	&	121	&	62	&	29087	&	29111.8	&	783.6	&	29063	&	29077.6	&	123	&	\bf{\underline{29051}}	&	29055.6	&	94.8	&	196	&	290.8	\\																													
Li	&	640	&	129	&	10	&	21368	&	21515.2	&	2930.2	&	21365	&	21435.3	&	240.8	&	\bf{\underline{21243}}	&	21330.9	&	254.6	&	265	&	519.6	\\																													
Li	&	720	&	145	&	11	&	24777	&	24834.4	&	4444.1	&	24578	&	24664.5	&	330	&	\bf{\underline{24486}}	&	24557.4	&	315.5	&	327	&	642.5	\\																													
Li	&	760	&	153	&	78	&	35190	&	35206.7	&	1669.5	&	35175	&	35182.9	&	242.6	&	\bf{\underline{35166}}	&	35169.6	&	143.2	&	284	&	427.2	\\																													
Li	&	800	&	161	&	11	&	27350	&	27482.7	&	7422.5	&	27284	&	27362.8	&	457.8	&	\bf{\underline{27238}}	&	27307.7	&	363.8	&	263	&	626.8	\\																													
Li	&	840	&	169	&	86	&	37878	&	37895	&	1881.6	&	37860	&	37871.3	&	298.3	&	\bf{\underline{37859}}	&	37871.4	&	222.8	&	302	&	524.8	\\																													
Li	&	880	&	177	&	11	&	30637	&	30766.5	&	10824.4	&	30568	&	30665.7	&	571.3	&	\bf{\underline{30483}}	&	30580.2	&	444.7	&	402	&	846.7	\\																													
Li	&	960	&	193	&	11	&	32873	&	32958	&	14578.7	&	32763	&	32888.3	&	742.6	&	\bf{\underline{32656}}	&	32717.1	&	515.6	&	290	&	805.6	\\																													
Li	&	1040	&	209	&	11	&	35991	&	36083	&	20634.3	&	35976	&	36044.6	&	843.5	&	\bf{\underline{35885}}	&	35919.1	&	554.6	&	440	&	994.6	\\																													
Li	&	1120	&	225	&	11	&	38690	&	38749	&	22831.5	&	38727	&	38785.8	&	1187.2	&	\bf{\underline{38669}}	&	38687.8	&	566.9	&	397	&	963.9	\\																													
Li	&	1200	&	241	&	11	&	41534	&	41621.2	&	25868.3	&	41563	&	41622.9	&	1269.8	&	\bf{\underline{41461}}	&	41490.7	&	582.8	&	411	&	993.8	\\													\hline	\end{tabular}

 \end{table}  
\end{landscape}

\end{document}